\documentclass[sigconf]{acmart}
\usepackage{array}
\usepackage{subfigure} 
\usepackage{graphicx} 
\usepackage{amsmath}
\usepackage{amsthm}
\usepackage{eucal}   
\usepackage{multicol}
\usepackage{multirow}
\usepackage{booktabs}
\usepackage{scalerel}
\usepackage{caption}
\usepackage{bm}
\usepackage{bbm}
\usepackage{graphicx}
\usepackage{algorithm}
\usepackage{algorithmic}
\usepackage{amssymb}
\usepackage{float}
\usepackage{stfloats}

\usepackage{balance}
\newtheorem{theorem}{Theorem}

\AtBeginDocument{%
  \providecommand\BibTeX{{%
    \normalfont B\kern-0.5em{\scshape i\kern-0.25em b}\kern-0.8em\TeX}}}

\copyrightyear{2022} 
\acmYear{2022} 
\setcopyright{acmlicensed}
\acmConference[CIKM '22]{Proceedings of the 31st ACM International Conference on Information and Knowledge Management}{October 17--21, 2022}{Atlanta, GA, USA}
\acmBooktitle{Proceedings of the 31st ACM International Conference on Information and Knowledge Management (CIKM '22), October 17--21, 2022, Atlanta, GA, USA}
\acmPrice{15.00}
\acmDOI{10.1145/3511808.3557333}
\acmISBN{978-1-4503-9236-5/22/10}

\begin{document}

\title{GDOD: Effective Gradient Descent using Orthogonal Decomposition for Multi-Task Learning}

\author{Xin Dong}
\authornote{Both authors contributed equally to this research.}
\orcid{1234-5678-9012}
\affiliation{%
  \institution{Ant Group}
  \city{Shanghai}
  \country{China}
}
\email{zhaoxin.dx@antgroup.com}

\author{Ruize Wu}
\authornotemark[1]
\affiliation{%
  \institution{Ant Group}
  \city{Hangzhou}
  \country{China}
  }
\email{kezhui.wrz@antgroup.com}

\author{Chao Xiong}
\affiliation{%
  \institution{Ant Group}
  \city{Shanghai}
  \country{China}
}
\email{xc272640@antgroup.com}

\author{Hai Li}
\affiliation{%
 \institution{Ant Group}
  \city{Shanghai}
  \country{China}
 }
\email{tianshu.lh@antgroup.com}

\author{Lei Cheng}
\affiliation{%
  \institution{Ant Group}
  \city{Hangzhou}
  \country{China}
  }
\email{lei.chenglei@antgroup.com}

\author{Yong He}
\affiliation{%
  \institution{Ant Group}
  \city{Hangzhou}
  \country{China}
  }
\email{heyong.h@antgroup.com}

\author{Shiyou Qian}
\affiliation{%
  \institution{Shanghai Jiao Tong University}
  \city{Shanghai}
  \country{China}
  }
\email{qshiyou@sjtu.edu.cn}

\author{Jian Cao}
\affiliation{%
  \institution{Shanghai Jiao Tong University}
  \city{Shanghai}
  \country{China}
  }
\email{cao-jian@cs.sjtu.edu.cn}

\author{Linjian Mo}
\authornote{Corresponding author.}
\authornotemark[0]
\affiliation{%
  \institution{Ant Group}
  \city{Shanghai}
  \country{China}
 }
\email{linyi01@antgroup.com}

\renewcommand{\shortauthors}{Xin Dong et al.}

\begin{abstract}
Multi-task learning (MTL) aims at solving multiple related tasks simultaneously and has experienced rapid growth in recent years.
However, MTL models often suffer from performance degeneration with negative transfer due to learning several tasks simultaneously.
Some related work attributed the source of the problem is the conflicting gradients.
In this case, it is needed to select useful gradient updates for all tasks carefully.
To this end, we propose a novel optimization approach for MTL, named GDOD, which manipulates gradients of each task using an orthogonal basis decomposed from the span of all task gradients.
GDOD decomposes gradients into task-shared and task-conflict components explicitly and adopts a general update rule for avoiding interference across all task gradients. 
This allows guiding the update directions depending on the task-shared components.
Moreover, we prove the convergence of GDOD theoretically under both convex and non-convex assumptions. 
Experiment results on several multi-task datasets not only demonstrate the significant improvement of GDOD performed to existing MTL models but also prove that our algorithm outperforms state-of-the-art optimization methods in terms of \emph{AUC} and \emph{Logloss} metrics.
\end{abstract}

\ccsdesc[500]{Computing methodologies~Multi-task learning}
\ccsdesc[500]{Computing methodologies~Transfer learning}

\keywords{multi-task learning, orthogonal decomposition, gradient conflict}


\maketitle

\section{Introduction}
Multi-task learning (MTL) aims to build a shared model that learns multiple related tasks simultaneously.
Compared to single-task learning, it can significantly improve learning efficiency and prediction accuracy through knowledge sharing between tasks~\cite{caruana1997multitask}. 
This allows MTL models to deploy to a wide range of real-world applications, such as computer vision~\cite{liu2019end}, natural language processing~\cite{collobert2008unified}, online recommendation and advertising systems~\cite{tang2020progressive}.
Recently, MTL has acted as a regularizer during network learning, leading to more meaningful neural representations and better generalization~\cite{subramanian2018learning}.

In practice, the training process of the MTL network is not always ideal.
Since the competition of shared parameter updates may harm individual tasks.
The MTL approach often leads to networks that accurately improve the performance of a subset of the tasks, while the rest suffer, a phenomenon referred to as \textit{negative transfer} or \textit{destructive interference}~\cite{ruder2017overview}.
Minimizing the negative transfer is a key goal for MTL models.
To mitigate this problem, several works~\cite{standley2020tasks,zamir2018taskonomy} opted to cluster tasks into groups based on prior beliefs about their similarity or relatedness.

Alternatively, some related work attributed the source of the problem to the gradient conflict~\cite{yu2020gradient}.
Several approaches have been proposed to minimize conflict between the updates across multiple tasks.
In this context, we split these approaches into two categories.
On one hand, gradients with different magnitudes lead to parameter updating dominated by a subset of tasks.
Consequently, a number of approaches have been developed to tune a set of task-weighting parameters to balance relative gradient magnitudes for different tasks~\cite{chen2018gradnorm,kendall2018multi}.
However, these approaches do not solve the problem of task gradients canceling out due to them pointing towards different directions.
On the other hand, some approaches~\cite{sener2018multi,yu2020gradient,liu2021conflict} find a common gradient descent direction for all tasks so that they do not cancel each other.
However, such solutions either cannot distinguish the conflicting gradients explicitly or cannot mitigate conflicting gradients completely.

Here, we argue the gradient conflict problem and reveal an advanced problem: how to distinguish the conflicting gradients and mitigate their impact on each task.
For this purpose, we instead present an approach that straightly manipulates gradients and mitigates the interference across tasks.
Specifically, we decompose the gradient by the orthogonal basis in the subspace spanned by all task per-example gradients. 
We analyze the updates of each task according to its impact on other tasks. 
So that each task gradients can be decomposed into two components: 1) task-shared component which is helpful for all tasks; and 2) task-conflict component which interferes with other tasks.
Only the task-shared component is used to update the network.
To achieve a tractable approach, we introduce a novel and robust algorithm, named GDOD, to estimate the subspace spanned by all task gradients and decompose each task update appropriately. 
As a result, we can integrate our approach with existing MTL models.
To evaluate the performance of GDOD, we conduct extensive experiments on three available public multi-task datasets and a large-scale industrial dataset. 
Consequently, GDOD guarantees convergence in theory and outperforms other state-of-the-art optimization methods across all datasets in experiments.

In light of the above background, the main contributions of this paper are the following:
\begin{itemize}
\item We propose an optimization approach, named GDOD, to manipulate each task gradient using an orthogonal decomposition built from the span of all task gradients.
GDOD decomposes gradients into task-shared and task-conflict components explicitly and adopts a general update rule for avoiding interference across all task gradients. 
\item We prove the convergence of GDOD theoretically under both convex and non-convex assumptions.
\item We conduct extensive experiments on several multi-task datasets to evaluate the effectiveness of GDOD.
Experiment results not only demonstrate the significant improvement of GDOD performed to existing MTL models but also outperform state-of-the-art optimization methods across all datasets in terms of AUC metric.
\end{itemize}

\section{Related Work}
Efficient multi-task learning models and optimization approaches of MTL models are two research areas related to our work. 

\subsection{Multi-Task Learning Models} 
The learning conception of MTL that modeling the shared representation for related tasks brings many benefits. 
However, MTL may suffer from negative transfer due to task conflicts as parameters are straightforwardly shared between tasks.
To deal with task conflicts, many works design different network architectures that allow optimal knowledge sharing between tasks.

Cross-stitch network~\cite{misra2016cross} and sluice network~\cite{ruder2017sluice} propose to learn weights of linear combinations to fuse representations from different tasks selectively.
However, the shared representations are combined with the same static weights in these models and the negative transfer is not addressed.
More studies apply the gate structure and attention network for representation fusion.
MOE~\cite{jacobs1991adaptive} splits the shared bottom layer into experts and combines experts through a gating network.
MMoE~\cite{ma2018modeling} and PLE~\cite{tang2020progressive} extend MOE to utilize different gate nets to aggregate experts for each task.
Similarly, MARN~\cite{zhao2019multiple} employs multi-head self-attention to learn different representations with different feature sets for each task.
However, none of the above works has explicitly addressed the issues of joint optimization of shared representation learning.

There are also some works utilizing neural architecture search (NAS) approaches to find a good MTL network architecture.
SNR framework~\cite{ma2019snr} controls connections between sub-networks by binary random variables and applies NAS~\cite{zoph2016neural} to search for the optimal structure.
Similarly, Gumbel-matrix routing framework~\cite{maziarz2019gumbel} learns to route MTL models formulated as a binary matrix with the Gumbel-Softmax trick. Moreover, \cite{rosenbaum2017routing} models the routing process as MDP and employs MARL~\cite{shoham2003multi} to train the routing network.
In contrast to these methods, we propose an approach to address the negative transfer problem in multi-task learning that allows us to learn the tasks simultaneously without the need for specific network design. 

\subsection{Optimization Methods in MTL}
Similar to our work, several prior researchers utilize some optimization methods to address the negative transfer problem in multi-task learning.
A very common solution is to balance the impact of individual tasks on the training of the network by adaptively weighting the task-specific losses or gradients.
There have been some studies developing a set of task-weighting parameters to balance the training procedure.
Uncertainty Weights~\cite{kendall2018multi} devises a weighting method dependent on the homoscedastic uncertainty inherently linked to each task.
These weights for each loss function are trained together with the MTL model parameters. 
GradNorm~\cite{chen2018gradnorm} reduces the task imbalances by weighting task losses so that their gradients are similar in magnitude.
There are several methods dynamically weighting the loss functions of tasks by the learning speed.
~\cite{liu2019end} and ~\cite{liu2019loss} explicitly set a weight to a task loss using a ratio of the current loss to the previous loss.
However, these loss weighting methods do not work well all the time in practice. 
Moreover, the formulation design of the weighing calculation is generally empirical and lacks theoretical derivation. 

There have also been some optimization methods to improve MTL performance by mitigating conflicting gradients.
The problem of conflicting gradients has been previously explored in multi-task learning as well as continual learning.
\cite{du2018adapting} and \cite{dery2021auxiliary} choose to ignore the gradients of auxiliary tasks if the direction is not similar to the main task.
\cite{riemer2018learning} overcomes catastrophic forgetting by maximizing the dot product between task gradients.
MGDA~\cite{desideri2012multiple} employs the condition of the Pareto stationary point for multi-objective optimization. 
It finds a linear combination of gradients that reduces every loss function simultaneously.
PCGrad~\cite{yu2020gradient} projects conflicting gradients to each other, which achieves a similar simultaneous descent effect as MGDA.
CAGrad~\cite{liu2021conflict} looks for an update vector that maximizes the worst local improvement of any objective in a neighborhood of the average gradient.
And the performance can shift from GD-like to MGDA-like by a hyper-parameter.
These methods deal with the gradient decent independent of the model structure and can be combined with normal optimizers such as SGD and Adam.
However, the above methods either cannot distinguish the conflicting gradients explicitly or cannot mitigate conflicting gradients completely.

\section{Multi-Task Learning using GDOD}
In this section, to realize effective gradient descent, we present a novel optimization approach that mitigates conflicting gradients across all tasks.

\begin{figure}
\centering
\captionsetup{font=bf}
\subfigure[Gradient Descent]{
        \includegraphics[width=82pt, height=88pt, trim=0pt 0pt 0pt 0pt, clip]{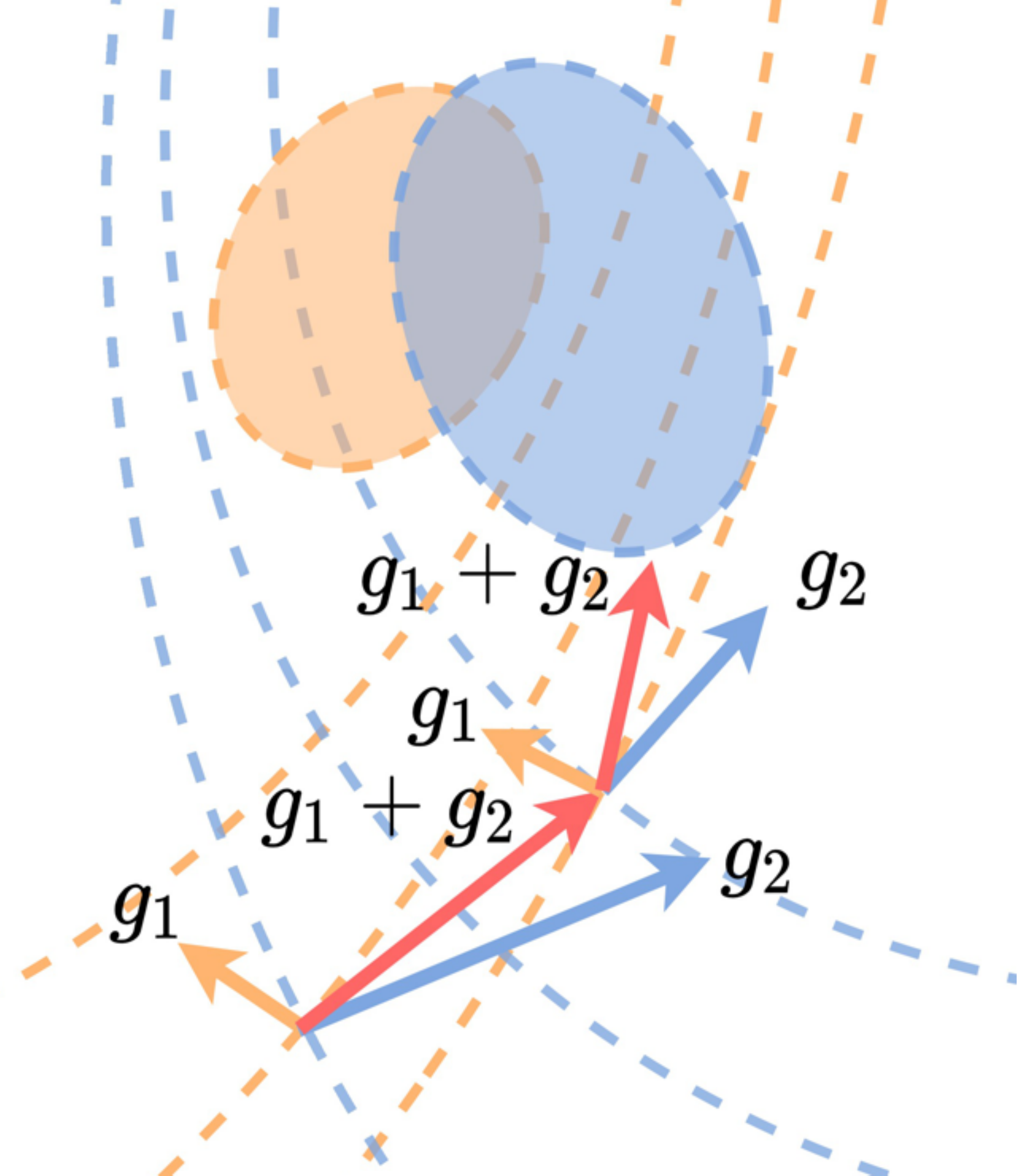}}
\subfigure[GDOD]{
    \includegraphics[width=140pt, height=88pt, trim=0pt 5pt 10pt 0pt, clip]{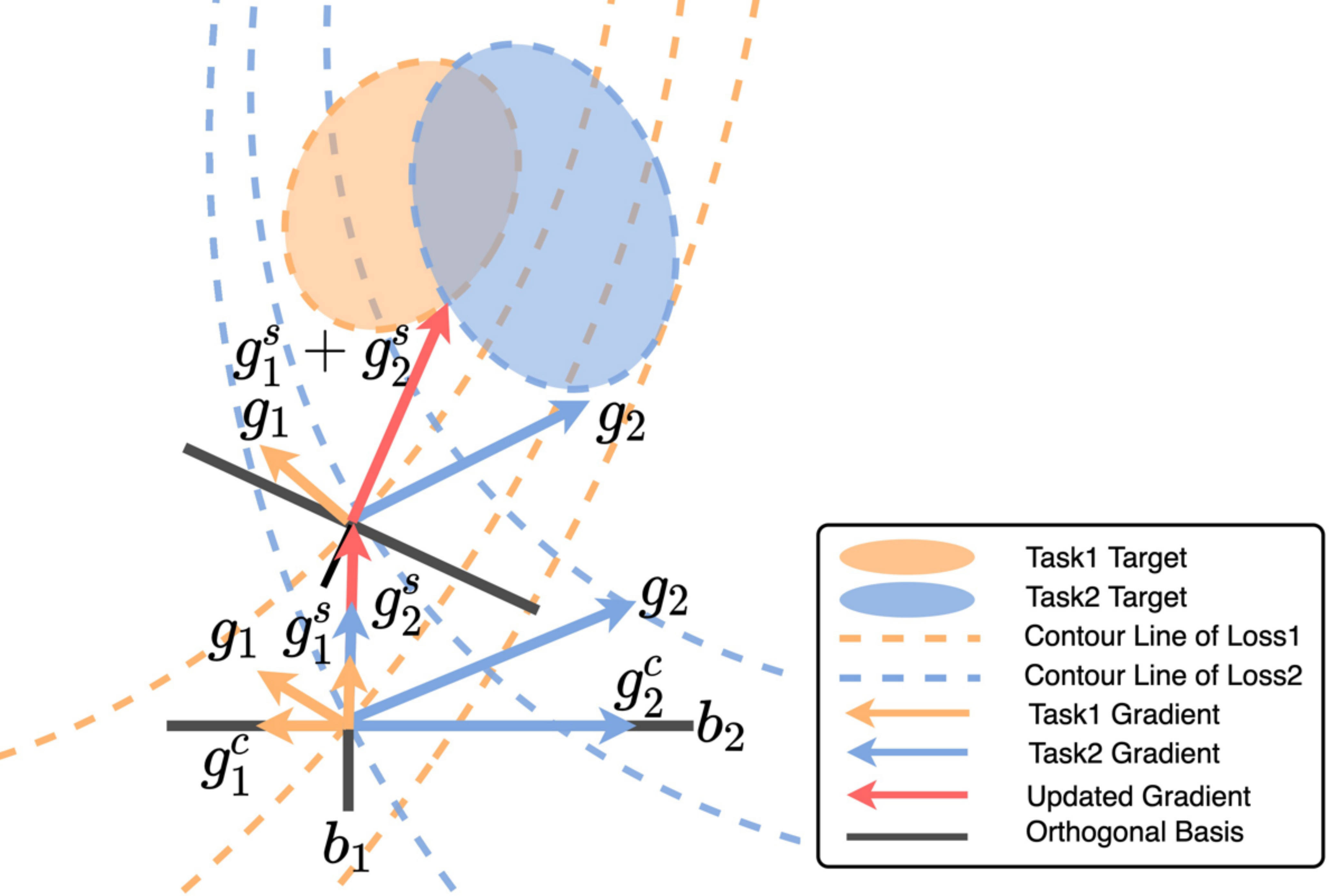}}
\caption{The overview of two optimization methods.
(a) Normal gradient descent.
(b) Gradient manipulation example in the 2-D plane with GDOD.}
\label{fig:overview}
\end{figure}

\subsection{Preliminaries: Notation and Problem}
Consider an MTL model with $K$ different tasks that we want to learn simultaneously.
For simplicity, we assume that they share an input space $\mathcal{X}$ and a collection of task spaces $\{\mathcal{Y}^i\}_{i \in [K]}$.
Each of the inputs in $\mathcal{X}$ is associated with a set of labels for all $K$ tasks, forming a large dataset of i.i.d. data points $\{x_j, y_j^1,...,y_j^K\}_{j\in [N]}$ of $N$ observations.
The MTL model can be divided into two parts: a backbone and multiple heads.
The backbone contains the shared parameters $\theta$ and transforms the input $\mathcal{X}$ into a shared representation.
Next, this representation is fed to each task-specific head to produce output.
We consider the empirical loss for each task $i$ as $\mathcal{L}_i(\theta, \theta_i) = \sum_{j \in N}l_i^j$, where $l_i^j$ is the loss for the $j$-th observation and $\theta_i$ are task-specific head parameters.
We aim to find both the backbone parameters $\theta$ and the task-specific parameters $\theta_i$ to minimize the loss $\mathcal{L}_i$.
More formally, for a set of $K$ tasks, the final goal is to minimize the multi-task loss as:
\begin{equation}
    \mathop{min}\limits_{\theta, \theta_1,...,\theta_K} \sum_{i=1}^{K}\mathcal{L}_i(\theta,\theta_i). 
\end{equation}

During training, each task competes for updating backbone parameters $\theta$ to minimize the loss for each individual task, this leads to the occurrence of negative transfer~\cite{ruder2017overview}.
Therefore, we focus on the learning of the backbone parameters $\theta$ in this work.
The task-specific parameters $\theta_i$ are updated by normal gradient descent.
For simplicity, we omit the task-specific parameters $\theta_i$ in loss $\mathcal{L}_i$.
Generally, the gradient of the loss with respect to a particular shared parameter $\theta$ from task $i$ is $g_i=\nabla_{\theta}\mathcal{L}_i(\theta)$.
Using gradient descent to minimize the multi-task loss, we obtain the update rule for a task-shared parameter $\theta$ as:
\begin{equation}
    \theta = \theta -  \gamma \sum_{i=1}^K \nabla_{\theta}\mathcal{L}_i(\theta). 
\end{equation}
The above expression shows that the overall success of an MTL model is dependent on the individual task gradients and their relationship to each other.
However, task gradients might cancel each other.
Or a subset of tasks might dominate the gradients and point towards a direction that does not improve any of the individual tasks.

In this paper, we provide an approach that straightly manipulates gradients to mitigate the conflicting gradient problem in MTL models.
Our approach decomposes each gradient by an orthogonal basis which is a subspace spanned by all
task per-example gradients.
It decomposes the gradients of all the tasks into two components, and only the helpful component is used to update the model parameters.


\subsection{Gradient Descent using Orthogonal Decompotision}
This section introduces a new optimization method to improve generalization for all tasks.
Figure~\ref{fig:overview} shows an illustration of the core idea of GDOD which manipulates each task gradients to maximize their usefulness to all tasks.
To ensure the convergence and stability of the optimization process, we modify the gradients for each task so as to minimize negative conflict with other task gradients.
Exactly, we decompose each task gradients into two components: 1) task-shared component which is helpful for all tasks; 
and 2) task-conflict component which interferes with other tasks.
Then, GDOD only utilizes the task-shared components from all tasks to update model parameters.

At each training step, suppose there are $m$ samples in a mini-batch of data. 
GDOD first collects the gradients of the losses with respect to $\theta$ for individual examples from all tasks as 
\begin{equation}
   \mathcal{M}^T = \{\{\nabla_{\theta}\mathcal{L}_1^j\}^T,...,\{\nabla_{\theta}\mathcal{L}_K^j\}^T\}_{\forall j \in [m]}, 
\end{equation}
where $\nabla_{\theta} \mathcal{L}_i^j \in \mathbb{R}^{m \times D}$ and $D$ is the dimension of the model parameters.
Then, we define a subspace $\mathcal{S}$ by the span of the gradient vectors in $\mathcal{M}^T$.
Any linear combination of each task gradients lies in this subspace, e.g., $g_i = \mathbb{E}(\nabla_{\theta}\mathcal{L}_i^j) \in \mathcal{S}$.
Note that, the dimension of the subspace $\mathcal{S}$ is $r$, which is the rank of the matrix $\mathcal{M} \in \mathbb{R}^{r \times D}(r \leq mK <<D)$.
Third, we project each task gradients onto $\mathcal{S}$.
This allows distinguishing between the directions of each task update which helps other task losses and those which have a negative impact.
Consequently, we can decompose each task gradients as 
\begin{equation}
    g_i = g_i^{sh} + g_i^{con}, 
\end{equation}
where $g_i^{sh} \in \mathcal{S}$ is the portion of the gradient that improves all task results and $g_i^{con} \in \mathcal{S}$ is the portion of the gradient that damages some task results.

In order to decompose each task gradient, we define an orthogonal basis for subspace $\mathcal{S}$ as $\{b_u\}_{u \in [r]}$.
On the orthogonal basis, we can measure whether the components of each gradient is agree or disagree with each other.
It is said that the two gradients agree along $b_u$ if and only if $sign(g_i^T \cdot b_u) = sign(g_j^T \cdot b_u)$ for task $i$ and $j$.
This mean that the gradient components $\{g_i^{sh}\}_{i \in [K]}$ refer to the projections of $\{g_i\}_{i \in [K]}$ onto the basis vectors where all task gradients agree.
On the contrary, $\{g_i^{con}\}_{i \in [K]}$ refer to the gradient components where the directions are disagree.
By this decomposition, $g_i^{sh}$ helps for all task, i.e., $(g_i^{sh})^T \cdot g_j \textgreater 0$ for any other task $j$, while $g_i^{con}$ interferes with some tasks, i.e., $(g_i^{con})^T \cdot g_k \textless 0$ for task $k$.

The remaining problem is how to select an orthogonal basis for the span of all task gradients.
There are multiple methods, such as Schmidt Decomposition, singular vector decomposition (SVD) and randomized approximate matrix decomposition~\cite{halko2011finding}, to obtain the basis $\{b_u\}_{u \in [r]}$ at each training time step.
The method is critical since the two components of $g_i^{sh}$ and $g_i^{con}$ are helpful or harmful depending on how they agree with the projection of all task gradients onto this basis.

As some work~\cite{mirzasoleiman2020coresets} proved that along the directions associated with singular vectors of a neural network Jacobian can generalize well.
In this paper, the subspace $\mathcal{S}$ is constructed from a mini-batch of all task gradients.
In brief, we select the singular vectors of the matrix of $\mathcal{M}$ as the basis.
We also compare the impact of different decomposition methods mentioned above in Section~\ref{sec:decom compa}.

\begin{algorithm}
\caption{GDOD Update At Each Training Step}
\label{alg:ogd}
\begin{algorithmic}[1]
\REQUIRE $\theta$, $\gamma$: model parameters shared with all tasks, learning rate
\REQUIRE $\mathcal{M}^T = \{\{\nabla_{\theta}\mathcal{L}_1^j\}^T,...,\{\nabla_{\theta}\mathcal{L}_K^j\}^T\}_{\forall j \in [m]}$: gradients with respect to $\theta$ for each task
\STATE $g_i = \frac{1}{m}\sum_{j=1}^m \nabla_{\theta}\mathcal{L}_i^j, \forall i \in [K]$ 
\STATE $B \leftarrow SVD\_Pro(\mathcal{M})$ 
\STATE $p_i = B(g_i)^T, \forall i \in [K]$
\STATE $p_i^{sh}, p_i^{con} = \left(1_{[p_1\odot...\odot p_K \geq 0]} \right) \odot p_i, \left(1_{[p_1\odot...\odot p_K \textless 0]} \right) \odot p_i, \forall i \in [K]$ // \textit{$\odot$ is the hadamard product operator}
\STATE $g_i^{sh}, g_i^{con} = (p_i^{sh})^TB, (p_i^{con})^TB, \forall i \in [K]$
\STATE \textbf{return} update $\theta = \theta - \gamma \sum_{i=1}^Kg_i^{sh}$
\end{algorithmic}
\end{algorithm}

We summarize the full update procedure in Algorithm~\ref{alg:ogd}.
Suppose the gradient for task $i$ is $g_i$.
GDOD proceeds as follows:
At each training step, it first obtains the orthogonal basis $B$ from the span of all task gradients $\mathcal{M}$ by SVD.
The $SVD\_Pro$ is the procedure of SVD and $B \in \mathbb{R}^{r \times D}$ is the non-zero and right-singular vectors of $\mathcal{M}$.
Secondly, GDOD decomposes each task gradient on the orthogonal basis $B$.
Thirdly, it obtains the helpful component $g_i^{sh}$ and harmful component $g_i^{con}$ for each task gradient respectively. 
Finally, it utilizes the helpful components from all tasks to update model parameters.

Algorithm~\ref{alg:ogd} shows that this procedure is simple to implement and ensures that the modified gradients we update for each task have no conflict with the other tasks in each training step.
Hence, GDOD mitigates the conflicting gradient problem in MTL models. 
Moreover, to reduce the computational complexity, the rank of the matrix $\mathcal{M}$ can be reduced by grouping gradients.
For example, the gradients are divided into different groups and an average pooling is performed on the gradients in the same group.
We also examine the impact of different dimensions of subspace $\mathcal{S}$ in Section~\ref{sec:size of S}.
Furthermore, by replacing the original gradients $\sum_{i=1}^Kg_i$ with the task-shared gradients $\sum_{i=1}^Kg_i^{sh}$, GDOD can be combined with any other gradient-based optimizer, such as SGD with momentum and Adam.

\subsection{Discussion}
Several gradient-based approaches have been proposed to manipulate each task gradients to obtain a new update vector and have shown improved performance on existing MTL models.
MGDA~\cite{desideri2012multiple} finds a linear combination of gradients that reduces
every loss function simultaneously.
It proposes to minimize the following combination of task gradients:  
\begin{align}
    {\rm min} \frac{1}{2}\Vert \sum_{i=1}^K w_ig_i \Vert^2, \quad s.t. \sum_{i=1}^K w_i=1 \quad {\rm and} \quad \forall i, \, w_i \geq 0. \label{equ:mgda}
\end{align}
From equation~\ref{equ:mgda}, MGDA seeks the linear combination of gradients that results in the smallest norm. 
Tasks that have larger gradients will become attenuated by MGDA.
For example, if an MTL model has two tasks where task 1 is under-optimized and task 2 is near a local optimum.
The model has a large gradient for task 1 and a relative small gradient for task 2.
In this situation, even if it is possible to improve task 1 a lot while not affecting the performance of task 2, MGDA may not take that move because the best improvement on task 1 is bounded by its improvement on task 2. 
This often causes slow improvement of MGDA in practice.

Moreover, PCGrad~\cite{yu2020gradient} projects conflicting gradients to the orthogonal direction of each other.
It sets a universal gradient similarity objective of zero for any two tasks explicitly. 
Consequently, if $g_i \cdot g_j \textless 0$, PCGrad projects these conflicting gradients to each other.
It replace $g_i$ ($g_j$) by its projection onto the normal plane of $g_j$ ($g_i$):
\begin{equation}
    g_i = g_i - \frac{g_i\cdot g_j}{||g_j||^2}g_j.
\end{equation}
It is not effective for the case of positive gradient similarities with $g_i \cdot g_j \geq 0$.
In fact, the two tasks share positive cosine similarities such that the precondition for PCGrad would never be satisfied. 
However, GDOD alters gradients more preemptively under both positive and negative cases, taking more proactive measurements in updating the gradient.

CAGrad~\cite{liu2021conflict} finds a linear combination of a new updated vector $d$ that is a linear combination of the original individual task gradients. 
It obtains the updated vector $d$ by solving the following optimization problem:
\begin{equation}
    \mathop{\rm max}\limits_{d \in R}\, \mathop{\rm min}\limits_{i \in [K]}<g_i, d>, \: s.t. \, \Vert d - g_0 \Vert \leq c \Vert g_0 \Vert.
\end{equation}
The difference between MGDA and CAGrad is that the new updated vector $d$ is searched around the $0$ vector for MGDA and $g_0$ (average gradient vector) for CAGrad. 
CAGrad chooses the average loss over all tasks as the main objective. 
Nevertheless, we find that CAGrad is not robust with different task weights in Section~\ref{sec:uncer task weight}.
For our method, we find an updated vector guided by the singular vectors of the Jacobian matrix. 
As some works~\cite{mirzasoleiman2020coresets} point out that using the principal vectors as directions of descent instead of the mean induces a more robust algorithm since the mini-batch average gradient is susceptible to outliers and skews from replicated data points.

\subsection{Theoretical Analysis of GDOD}
In this section, we analyze the convergence of GDOD with the following theorem.
\begin{theorem}
Let $\mathcal{L}(\theta^t)$ represents the full batch losses of all $K$ tasks at training step $t$.
Suppose the gradients $\{g_i\}_{i \in [K]}$ of all $K$ tasks are Lipschitz continuous with $L > 0$. 
Then, the GDOD update rule $\theta^{t+1} = \theta^t - \gamma \sum_{i=1}^Kg_i^{sh}$ with learning rate $\gamma \leq \frac{1}{L}$ will converge to either (1) the optimal value if $\mathcal{L}(\theta)$ is convex or (2) a stationary point if $\mathcal{L}(\theta)$ is non-convex.
\end{theorem}
\begin{proof}
According to the Lipschitz smoothness assumption, we obtain the following inequality:
\begin{align}
    \mathcal{L}(\theta^{t+1}) & \leq \mathcal{L}(\theta^t) + \nabla_{\theta}\mathcal{L}(\theta^t)^T(\Delta\theta)
    + \frac{1}{2} L \Vert \Delta\theta \Vert_2^2 \nonumber 
    \nonumber
\end{align}
Now, we can plug in the GDOD update by replacing $\Delta\theta = \theta^{t+1} - \theta^t = - \gamma \sum_{i=1}^Kg_i^{sh}$. 
We then obtain:
\begin{align}
\mathcal{L}(\theta^{t+1}) & \leq \mathcal{L}(\theta^t) + \gamma (\sum_{i=1}^K g_i)^T(-\sum_{i=1}^K g_i^{sh}) + \frac{1}{2} L \Vert \gamma \sum_{i=1}^K g_i^{sh} \Vert_2^2 \nonumber \\
    & = \mathcal{L}(\theta^t) + \gamma (\sum_{i=1}^K g_i^{sh} + \sum_{i=1}^K g_i^{con})^T(-\sum_{i=1}^K {g}_i^{sh}) \nonumber \\
    & \qquad + \frac{1}{2} L\gamma^2 \Vert \sum_{i=1}^K {g}_i^{sh} \Vert_2^2 \nonumber \\
    & = \mathcal{L}(\theta^t) - \gamma (\sum_{i=1}^K {g}_i^{sh})^2 - \gamma (\sum_{i=1}^K {g}_i^{con})^T (\sum_{i=1}^K {g}_i^{sh}) \nonumber \\
    & \qquad + \frac{1}{2} L\gamma^2 \Vert \sum_{i=1}^K {g}_i^{sh} \Vert_2^2 \label{equ:aa}  \\
    & = \mathcal{L}(\theta^t) - \gamma (\sum_{i=1}^K {g}_i^{sh})^2 + \frac{1}{2} L\gamma^2 \Vert \sum_{i=1}^K {g}_i^{sh} \Vert_2^2 \label{equ:bb}  \\
    & = \mathcal{L}(\theta^t) - (1 - \frac{1}{2} L\gamma)\gamma\Vert \sum_{i=1}^K {g}_i^{sh} \Vert_2^2 \nonumber
\end{align}
Note that in going equation~\ref{equ:aa} to~\ref{equ:bb} in the above proof, we use the fact that $(g_i^{con})^Tg_j^{sh} = 0$ for any two tasks $i$ and $j$ due to orthogonality.
We define the updated gradient at training step $t$ is $g_t^{sh} = \sum_{i=1}^K {g}_i^{sh}$.
Using $\gamma \leq \frac{1}{L}$, we know that 
\begin{equation}
    -(1 - \frac{1}{2} L\gamma) = \frac{1}{2} L\gamma -1 \leq \frac{1}{2}L(\frac{1}{L}) -1 = -\frac{1}{2}. \nonumber
\end{equation}
Plugging this into the last expression above, we can conclude the following:
\begin{align}
\mathcal{L}(\theta^{t+1}) & \leq \mathcal{L}(\theta^t) - \frac{1}{2}\gamma\Vert {g}_t^{sh} \Vert_2^2 \label{equ:cc} \\
    & \leq \mathcal{L}(\theta^t) \nonumber
\end{align}
Thus, the above theorem ensures that GDOD is minimizing $\mathcal{L}(\theta^t)$.
If $\mathcal{L}(\theta)$ is convex and differentiable, hence repeatedly applying GDOD process can reach the optimal value.

Assume $\mathcal{L}(\theta)$ is non-convex, using telescope sum to equation~\ref{equ:cc}, we have
\begin{align}
    \mathcal{L}(\theta^T) - \mathcal{L}(\theta^0) \leq - \frac{1}{2}\gamma\sum_{t=0}^{T-1} \Vert {g}_t^{sh} \Vert_2^2
\end{align}
Thus, we have:
\begin{align}
& \min_{0\leq{t}\leq{T}}{\Vert {g}^{sh}_t \Vert_2^2 }  \leq \frac{1}{T} \sum_{t=0}^{T-1} \Vert {g}^{sh}_t \Vert^2_2   \nonumber\\
&\leq \frac{2(\mathcal{L}(\theta^0) - \mathcal{L}(\theta^T))}{T\gamma}  \nonumber \\
& \leq \frac{2(\mathcal{L}(\theta^0) - \mathcal{L}^{*})}{T\gamma}
\end{align}
where $\mathcal{L}^{*}$ is the minimal function value.
Therefore, GDOD updating with gradients $g_t^{sh}$ can converge to a stationary point in $\mathcal{O}(\frac{1}{T})$ steps.
\end{proof}

Therefore, we prove GDOD converges to either the optimal value if $\mathcal{L}(\theta)$ is convex or a stationary point if $\mathcal{L}(\theta)$ is non-convex.

\section{Experiments}
In this section, we evaluate the performance and effectiveness of GDOD with four multi-task datasets from different domains. 
We first evaluate the performance of GDOD as well as several state-of-the-art optimization methods.
Then, we verify that GDOD is model-agnostic and can improve performance for any MTL models with shared parameters.
Finally, we present ablation experiments to explain the impact of hyper-parameter selection.

\begin{table}[htb]
  \centering
  \captionsetup{font=bf}
  \caption{The statistics of the four datasets.}
\renewcommand\arraystretch{1.0}
\begin{tabular}{ c | c | c | c | c}
\toprule
 \textbf{Dataset} & \textbf{Phase} &\textbf{Users} & \textbf{Items} & \textbf{Samples}\\
\toprule\toprule
 \multirow{2}{*}{BookCrossing} & Train & 92,792 & 239,029 & 919,824\\
   & Test & 42,194 & 99,404 & 229,956\\
\cline{1-5}
 \multirow{2}{*}{IJCAI-15}  & Train & 237,295 & 274,709 & 2,142,528 \\
 & Test & 106,023 & 127,772  & 544,025 \\
\cline{1-5}
 \multirow{2}{*}{Alipay Advertising} & Train & 7,579,571 & 1,098 & 14,298,291 \\
 & Test & 5,822,077 & 835 & 10,740,289 \\ 
\cline{1-5}
 \multirow{2}{*}{Census-Income} & Train & - & - & 199,523\\
 & Test & - & - & 99,762\\
\toprule
\end{tabular}
\label{tab: datasets}
\end{table}

\subsection{Datasets and Settings}
\subsubsection{Datasets}
We use three public multi-task datasets from different domains and a large-scale real-world advertising dataset to verify the effectiveness of GDOD. 
The statistics of the datasets are listed in Table~\ref{tab: datasets}.
These datasets are described as follows:

\begin{itemize}
\item \textbf{BookCrossing Dataset}~\cite{bookcrossing} collects user ratings in the Book-Crossing community. 
 It includes 278,858 users who provide 1,157,112 ratings about 271,379 books.
As suggested in the original paper \cite{bookcrossing}, we define the following two related prediction tasks based on this dataset: 1) predict whether a user has rated a book; 2) predict whether a rating score from a user on a book is higher than or equal to 9.
\item \textbf{IJCAI-15 Dataset}~\cite{ijcai} is collected from the E-commerce website Tmall.com.
It is a public dataset used in the IJCAI2015 repeat buyers prediction competition hosted by Alibaba Group.
It contains 241,093 users with 2,295,706 instances on 237,564 items.
We model two related prediction tasks involving CVR (Conversion Rate): 1) predict whether a user adds an item to his favourites after clicking; 2) predict whether a user buys an item after adding it to his favourites.
\item \textbf{Alipay Advertising Dataset} is collected over three months from user traffic logs of a commercial advertising system in the Alipay App. 
It contains 7,630,003 users who produce 25,038,580 samples about 1,120 advertisements. 
One CTR task to predict whether a user clicks an item and two CVR tasks similar with IJCAI-15 dataset are modeled.

\item \textbf{Census-Income(KDD) Dataset}~\cite{census_kdd} is a dataset extracted from the 1994 census database.
It contains 199,523 instances with 42 demographic and employment related features.
Given a person, we model six related prediction tasks based on this dataset contains: 1) predict whether the person's income exceeds \$50K; 2) predict whether the person’s marital status has never married; 3) predict whether the person's education level is at least college; 4) predict whether the person's employment status is full time; 5) predict whether the person's gender is male; and 6) predict whether the person's race is white.
\end{itemize}

\subsubsection{Comparative Optimization Methods} 
We compare GDOD with seven SOTA optimization methods.
\begin{itemize}
    \item \textbf{Adam} is used as the baseline to compute the performance gains of other methods.
    \item \textbf{Uncertainty Weights (Uncert)}~\cite{kendall2018multi} uses a joint likelihood formulation to derive task weights based on the intrinsic uncertainty in each task.
    \item \textbf{GradNorm}~\cite{chen2018gradnorm} 
    reduces the task imbalances by weighting task losses so that their gradients are similar in magnitude.
    \item \textbf{MGDA}~\cite{desideri2012multiple} applies a multiple-gradient descent algorithm for MTL. 
    It finds a linear combination of gradients that reduces every task loss simultaneously.
    \item \textbf{Gradient Regularization (GradReg)}~\cite{suteu2019regularizing} proposes a gradient regularization term that minimizes task interference by enforcing near orthogonal gradients.
    \item \textbf{PCGrad}~\cite{yu2020gradient} projects conflicting gradients to the orthogonal direction of each other, so that achieving a similar simultaneous descent effect.
    \item \textbf{CAGrad}~\cite{liu2021conflict} looks for an update gradient vector in the neighborhood of the average gradient that minimizes the average loss and leverages the worst local improvement of individual tasks.
\end{itemize}

\subsubsection{Baseline MTL models}  
We evaluate the effect of our GDOD with the following representative MTL models.
\begin{itemize}
    \item \textbf{Shared-Bottom}~\cite{caruana1997multitask}. Shared-Bottom shares the embedding layers and a low-level feature extraction layer (MLP) for all tasks and each task has its own task-specific high-level layers built on top of the shared layers.
    \item \textbf{Cross-Stitch}~\cite{misra2016cross}. It fuses the tower layers of tasks by linear transformation based on the Shared-Bottom model. 
    \item \textbf{MMOE}~\cite{ma2018modeling}. MMOE transforms the shared low-level layers into sub-networks and uses different gating networks for tasks to utilize different sub-networks. 
    \item \textbf{SNR}~\cite{ma2019snr}. SNR modularizes the shared low-level layers into parallel sub-networks and uses a transformation matrix multiplied by a scalar coding variable to learn their connections. 
    \item \textbf{PLE}~\cite{tang2020progressive}. PLE separates shared components and task-specific components and adopts a progressive routing mechanism to achieve more effective information sharing.
\end{itemize}

\subsubsection{Evaluation Metrics }
For fair comparisons, we employ \emph{AUC} and \emph{Logloss} as our evaluation metrics.
\begin{itemize}
\item \textbf{AUC} is the Area Under the ROC Curve over the test set.
It measures the goodness of order by ranking all the items with predicted CTR in the test set.
It is noticeable that a slightly higher AUC at 0.001-level is regarded as significant for CTR/CVR prediction tasks, which has been pointed out in existing works~\cite{cheng2016wide,guo2017deepfm,misra2016cross}.
Note that, the larger AUC shows better performance.
\item \textbf{Logloss} is the loss value on the test set.
The smaller Logloss means better performance.
\end{itemize}

\subsubsection{Implementation Details}
For all the baseline MTL models, there are trained by the Adam optimizer~\cite{kingma2014adam} with an initial learning rate of 1e-3. 
The mini-batch size is fixed to 256.
The embedding size of each sparse feature is set to 8. 
The hidden sizes of the two shared hidden layers in shared-bottom model are [256, 32].
The number of sub-networks/experts in SNR, PLE and MMoE is set to 8 and the hidden size of each sub-network/expert is 32.
There are two specific tower hidden layers with the size of [16, 1] for each task.

Moreover, the weights in GradReg is tuned in [1e-1, 1e-2, 1e-3, 1e-4].
For GradNorm, the hyper-parameter $\alpha$ is tuned in [0.5, 1.5].
The hyper-parameter $c$ in CAGrad is tuned in [0.1, 0.3, 0.5, 0.7, 0.9]. 
All hyper-parameters are settled with the best performance on each dataset.
For GDOD, the training examples are divided into $16$ groups in a mini-batch at each training step. 
For PCGrad, CAGrad and GDOD, they are combined with Adam by passing the computed update to replace the original gradient.
We repeat all experiments 5 times and report the averaged results.

\begin{table*}[htb]
  \centering
  \captionsetup{font=bf}
  \caption{Performance comparisons of different optimization methods. 
  The baseline MTL model is Shared-Bottom.
  \emph{Gain} measures the AUC improvement between Adam with other optimization methods.
  }
  \label{tab:commparison optimazation}
\renewcommand\arraystretch{1.0}
\resizebox{507pt}{59pt}{
\begin{tabular}{ c|c c | c c|c c | c c|c c | c c | c c }
 \toprule
 \multicolumn{1}{c}{\multirow{3}{*}{\textbf{Optimization Method}}}  &  \multicolumn{4}{|c|}{\textbf{BookCrossing} } & \multicolumn{4}{c|}{\textbf{IJCAI-15}} & \multicolumn{6}{c}{\textbf{Alipay Advertising}}  \\
 \cline{2-15}
  ~ & \multicolumn{2}{|c|}{Task1} & \multicolumn{2}{c}{Task2} & \multicolumn{2}{|c|}{Task1} & \multicolumn{2}{c}{Task2} & \multicolumn{2}{|c|}{Task1} & \multicolumn{2}{c}{Task2} & \multicolumn{2}{|c}{Task3} \\
 & AUC & Gain & AUC & Gain & AUC & Gain & AUC & Gain & AUC & Gain & AUC & Gain & AUC & Gain \\
 \hline
 \hline
Adam & 0.7838  & - & 0.7633  & - & 0.6968  & - & 0.7451  & - & 0.7505  & - & 0.7387  & - & 0.8237  & -  \\
Uncert & 0.7814  & -0.0024  & 0.7663  & 0.0030  & 0.6631  & -0.0337  & 0.7327  & -0.0124  & 0.7599  & 0.0094  & 0.7475  & 0.0088  & 0.8264  & 0.0027   \\
GradReg & 0.7696  & -0.0142  & 0.7466  & -0.0167  & 0.6775  & -0.0193  & 0.7321  & -0.0130  & 0.7635  & 0.0130  & 0.7513  & 0.0126  & 0.8273  & 0.0036   \\
GradNorm & 0.7857  & 0.0019  & 0.7677  & 0.0044  & 0.7125  & 0.0157  & 0.7477  & 0.0026  & 0.7649  & 0.0144  & 0.7505  & 0.0118  & 0.8238  & 0.0001   \\
MGDA & 0.7811 & -0.0027 & 0.7593 & -0.0040 & 0.7122 & 0.0154 & 0.7548 & 0.0097 & 0.7600 & 0.0095 & 0.7440 & 0.0053 & 0.8300 & 0.0063 \\
PCGrad & 0.7912 & 0.0074 & 0.7753 & 0.0120 & 0.7188 & 0.0220 & 0.7524 & 0.0073 & 0.7630 & 0.0125 & 0.7431 & 0.0044 & 0.8335 & 0.0098  \\
CAGrad & 0.7900 & 0.0062 & 0.7733 & 0.0100 & 0.7204 & 0.0236 & 0.7542 & 0.0091 & 0.7699 & 0.0194 & 0.7456 & 0.0069 & 0.8384 & 0.0147  \\
GDOD & 0.7922  & \textbf{0.0084}  & 0.7817  & \textbf{0.0184}  & 0.7268  & \textbf{0.0300}  & 0.7555  &\textbf{ 0.0104}  & 0.7723  & \textbf{0.0218}  & 0.7571  & \textbf{0.0184}  & 0.8390  & \textbf{0.0153}   \\
 \hline
\toprule
\end{tabular}
}
\label{tab: optimization}
\end{table*}

\begin{figure*}[h]
    \setlength{\abovecaptionskip}{1pt}
    \centering
    \captionsetup{font=bf}
    \subfigure[BookCrossing task1 loss.]{
        \includegraphics[width=120pt, trim=0pt 0pt 0pt 0pt, clip]{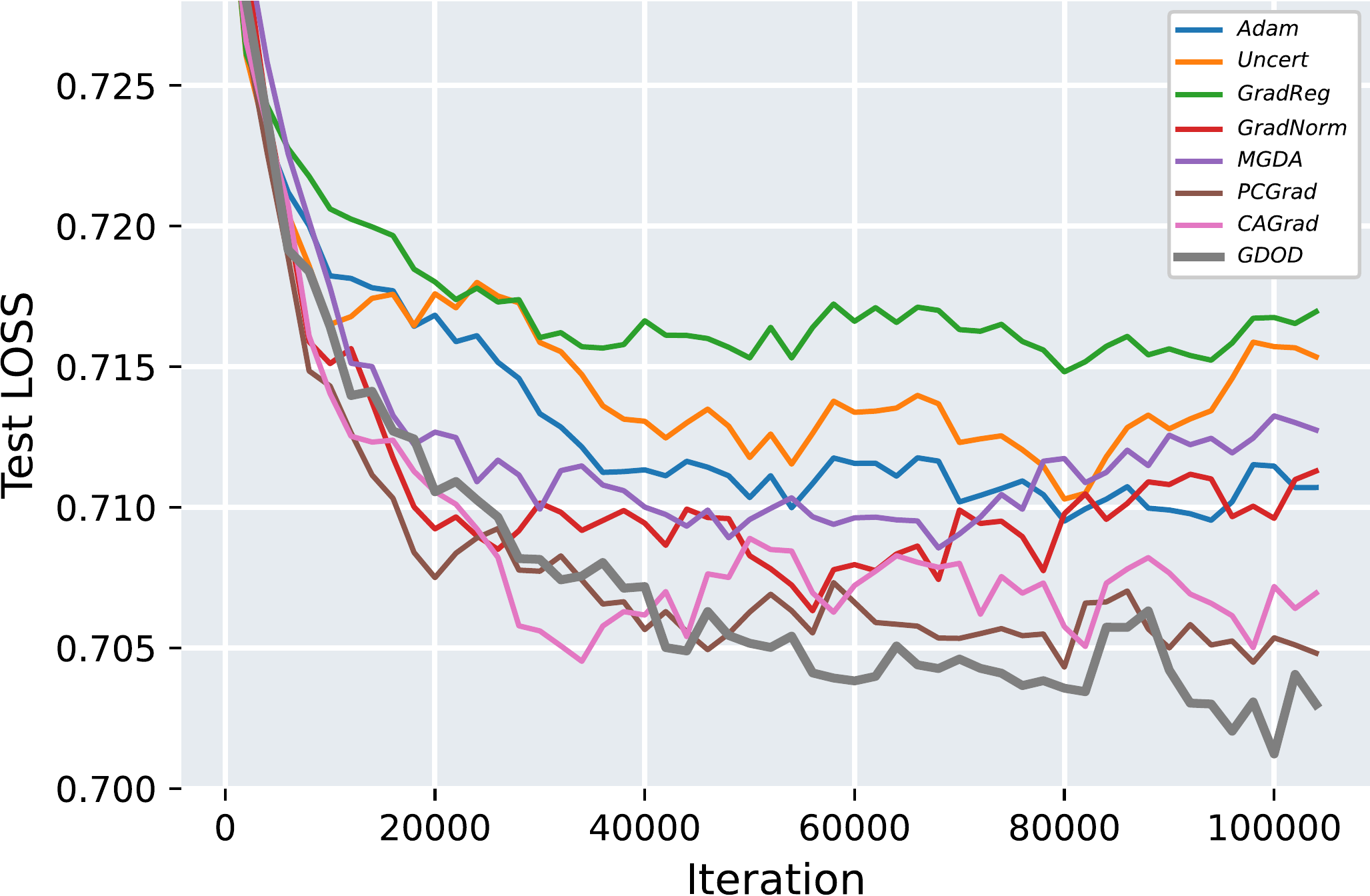}}
    \subfigure[BookCrossing task2 loss.]{
        \includegraphics[width=120pt, trim=0pt 0pt 0pt 0pt, clip]{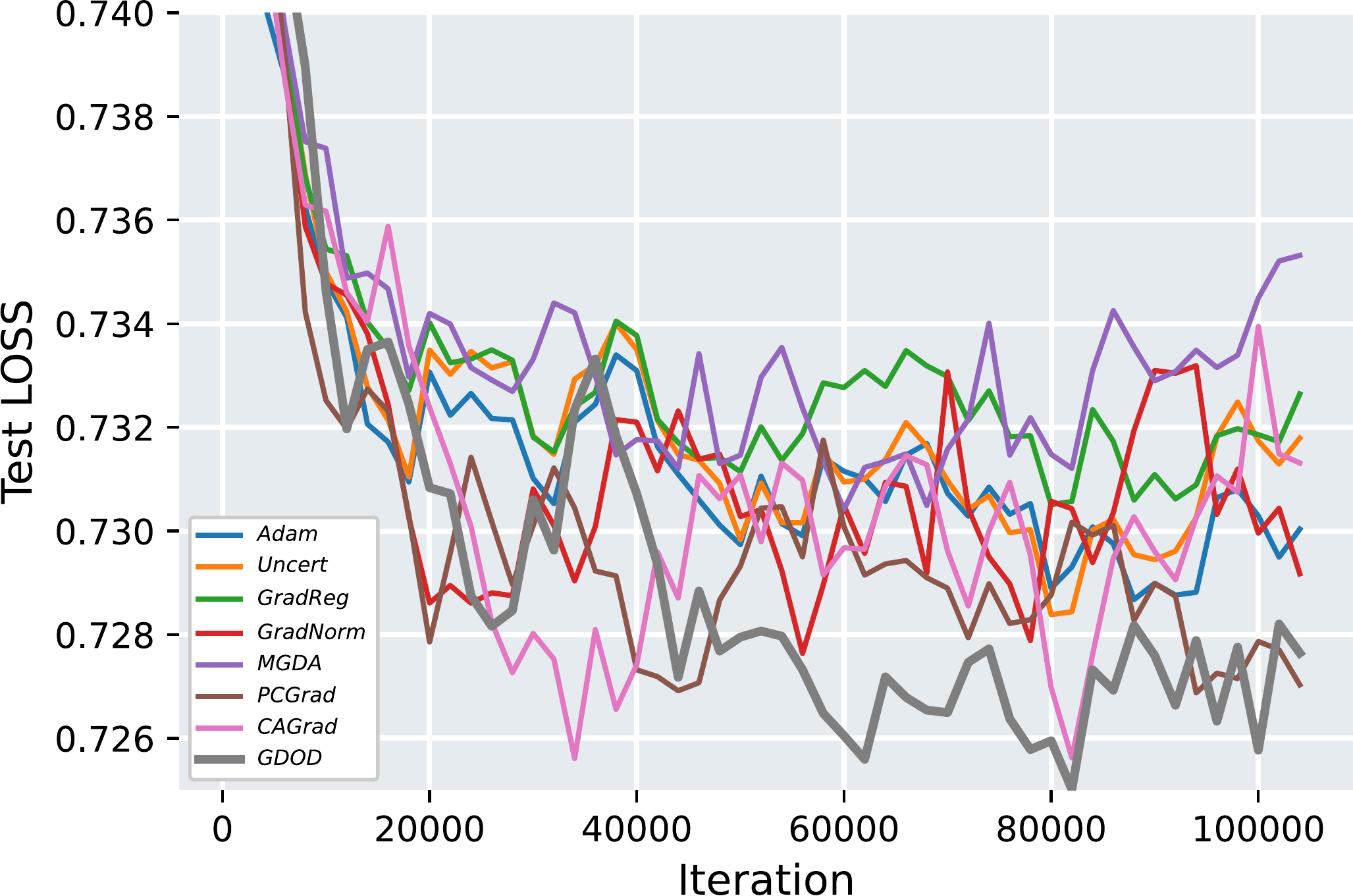}}
    \subfigure[IJCAI-15 task1 loss.]{
        \includegraphics[width=120pt, trim=0pt 0pt 0pt 0pt, clip]{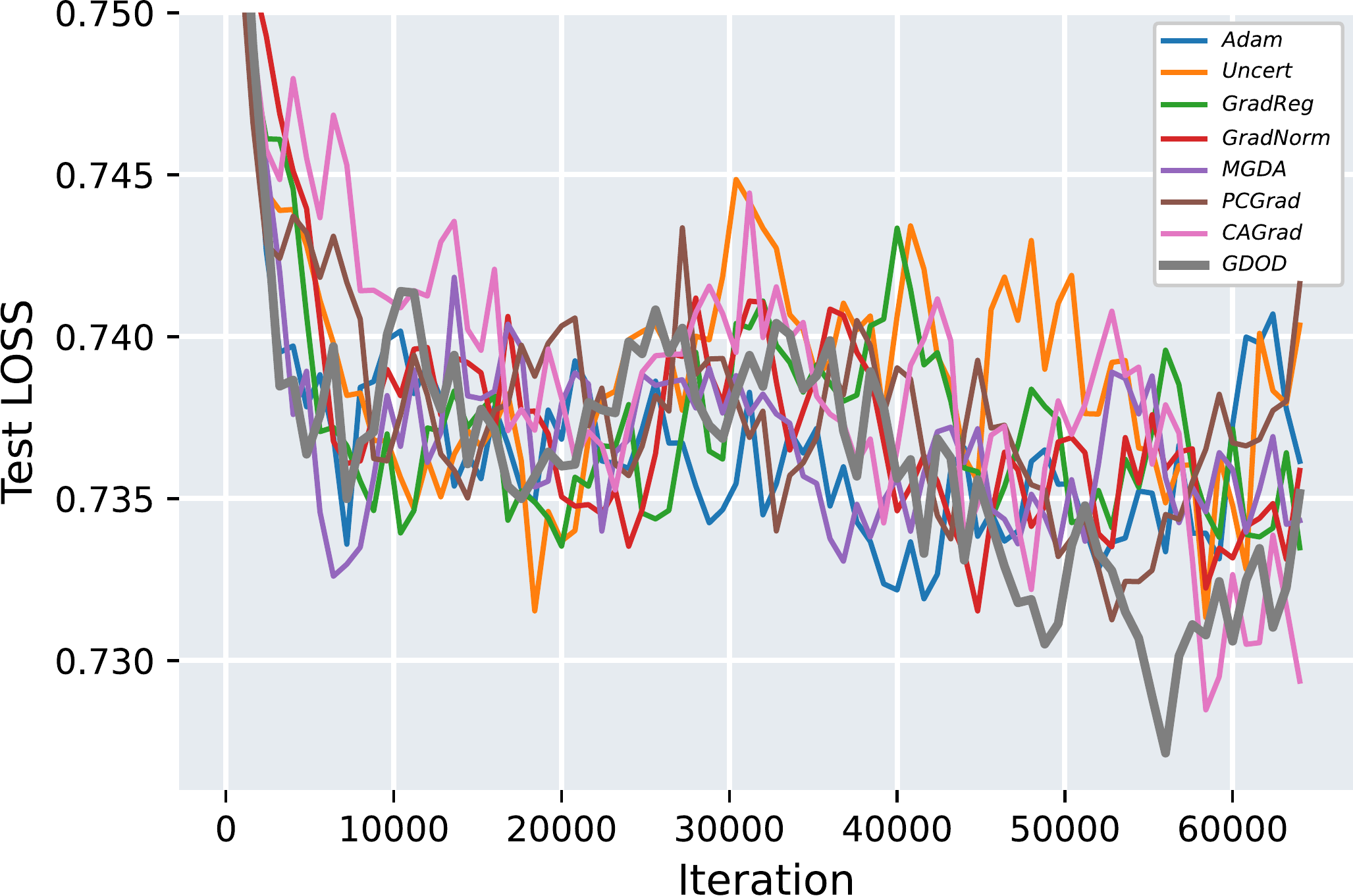}}
    \subfigure[IJCAI-15 task2 loss.]{
        \includegraphics[width=120pt, trim=0pt 0pt 0pt 0pt, clip]{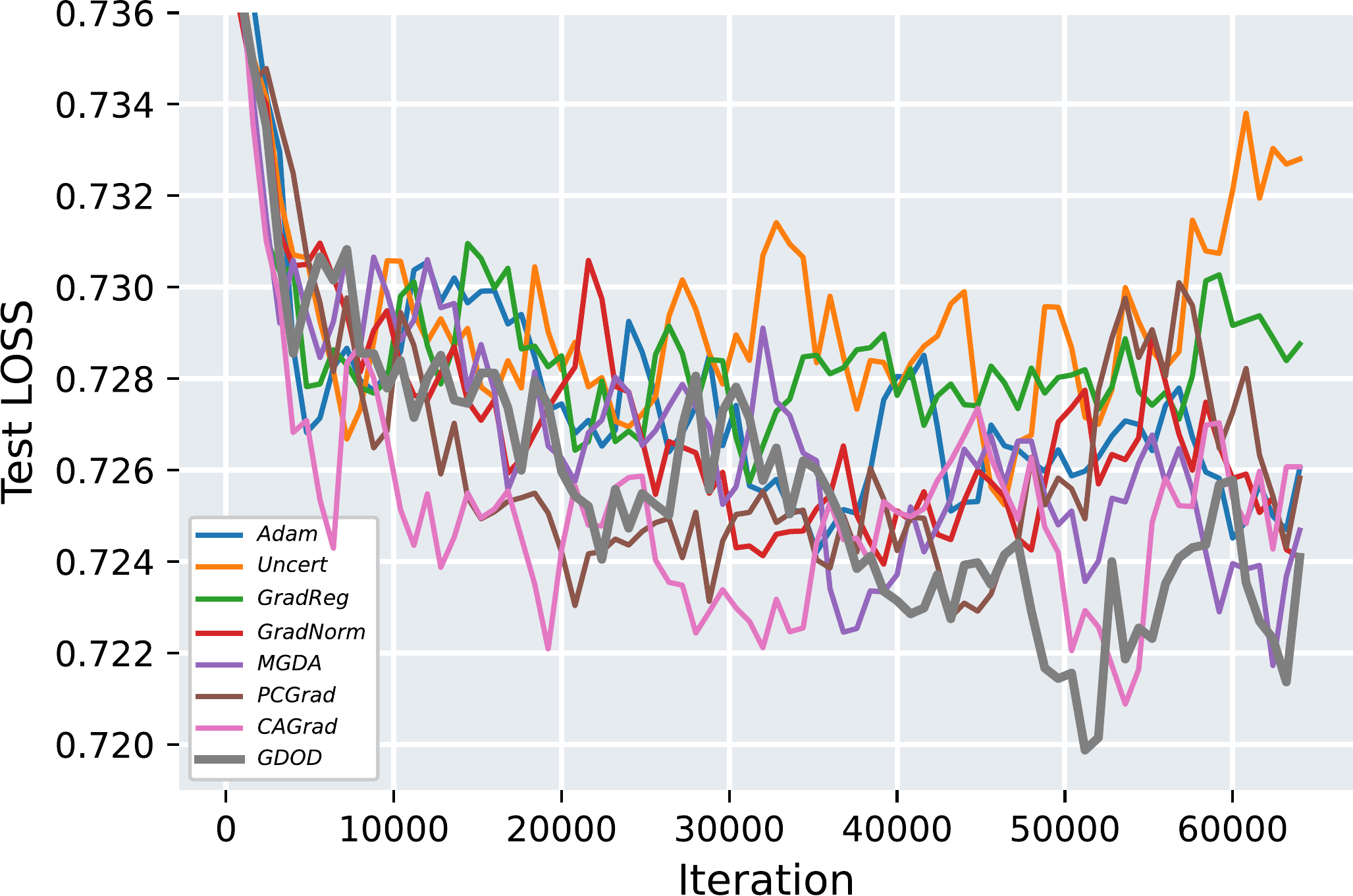}}
    \caption{Test loss comparisons about several optimization methods on BookCrossing and IJCAI-15 datasets.
    In all cases GDOD outperforms all other optimization methods.}
    \label{fig:Loss Comparisons}
\end{figure*}

\begin{figure*}[htb]
\setlength{\abovecaptionskip}{1pt}
    \centering
    \captionsetup{font=bf}
    \subfigure[Task1 on BookCrossing]{
        \includegraphics[width=119pt, trim=0pt 0pt 0pt 0pt, clip]{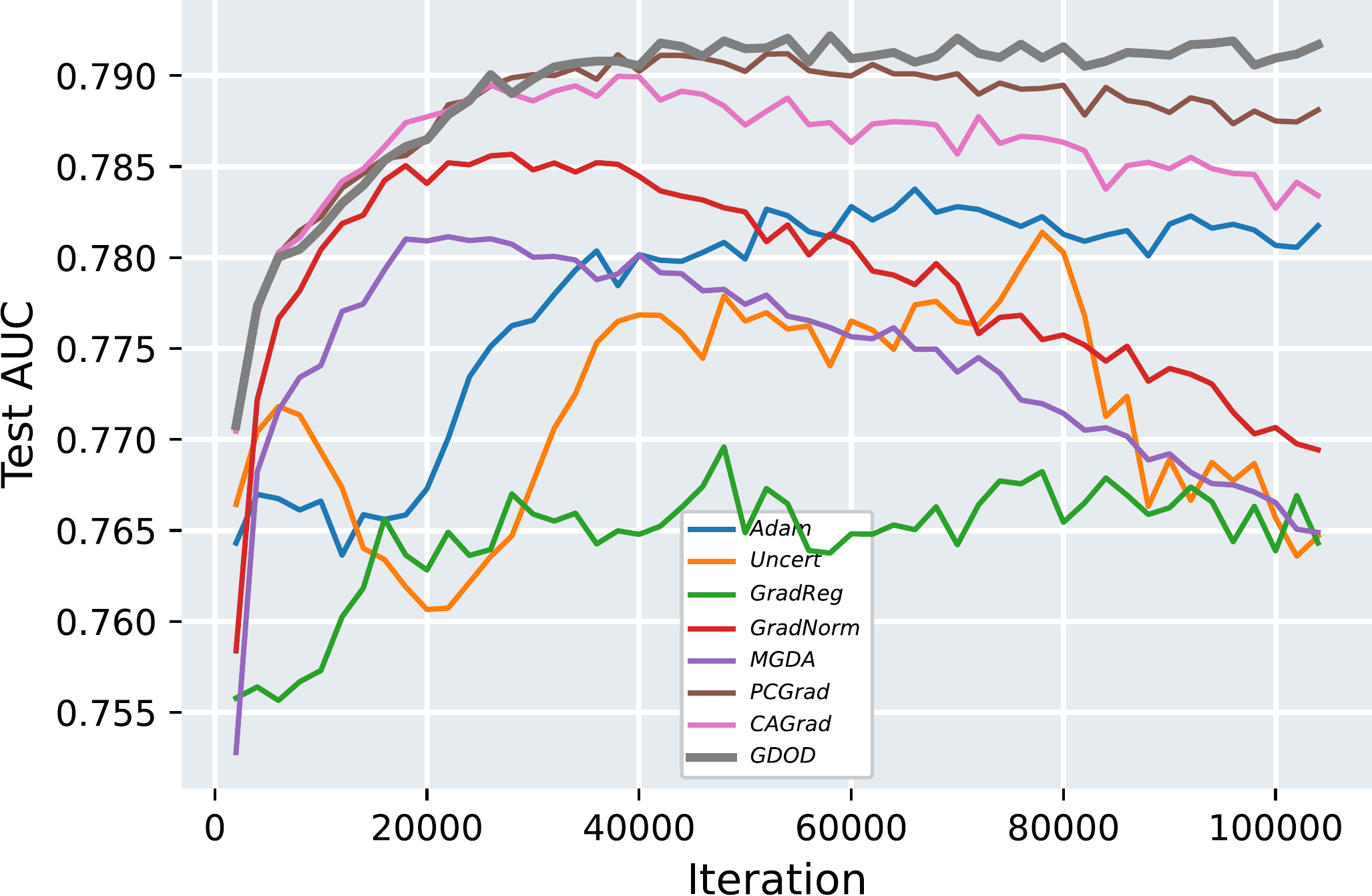}}
    \subfigure[Task2 on BookCrossing]{
        \includegraphics[width=119pt, trim=0pt 0pt 0pt 0pt, clip]{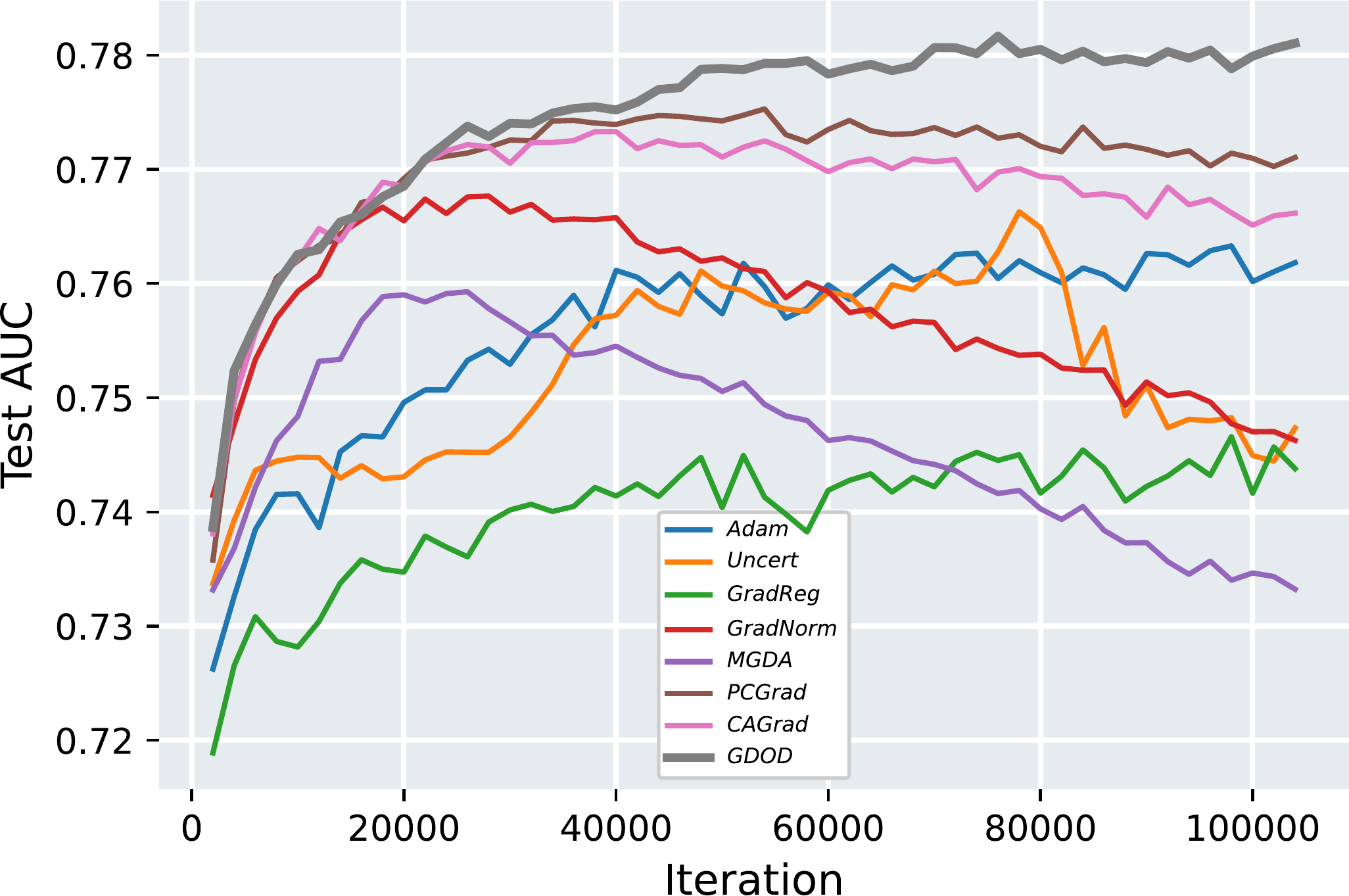}}
    \subfigure[Task1 on IJCAI-15]{
        \includegraphics[width=119pt, trim=0pt 0pt 0pt 0pt, clip]{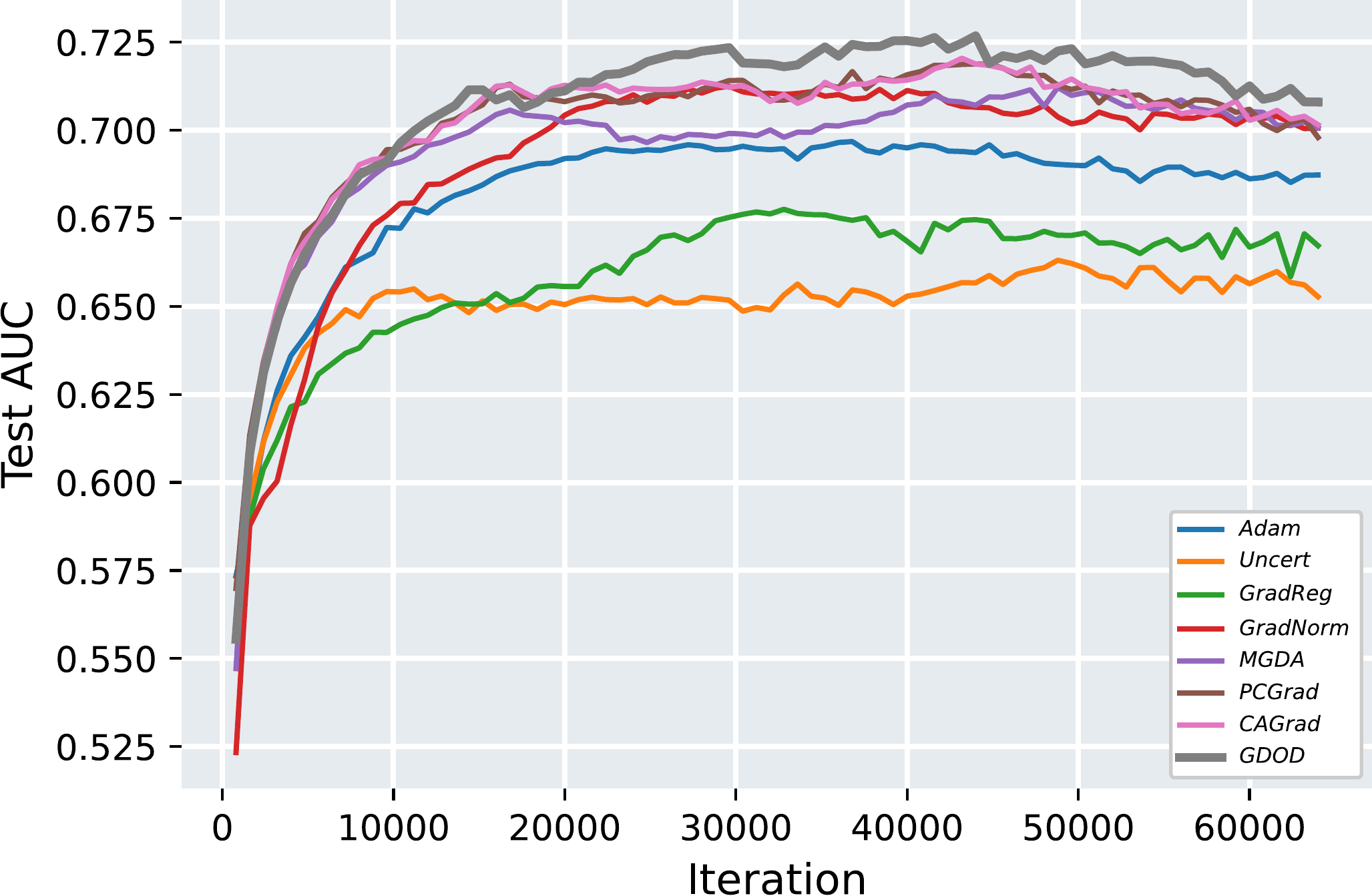}}
    \subfigure[Task2 on IJCAI-15]{
        \includegraphics[width=119pt, trim=0pt 0pt 0pt 0pt, clip]{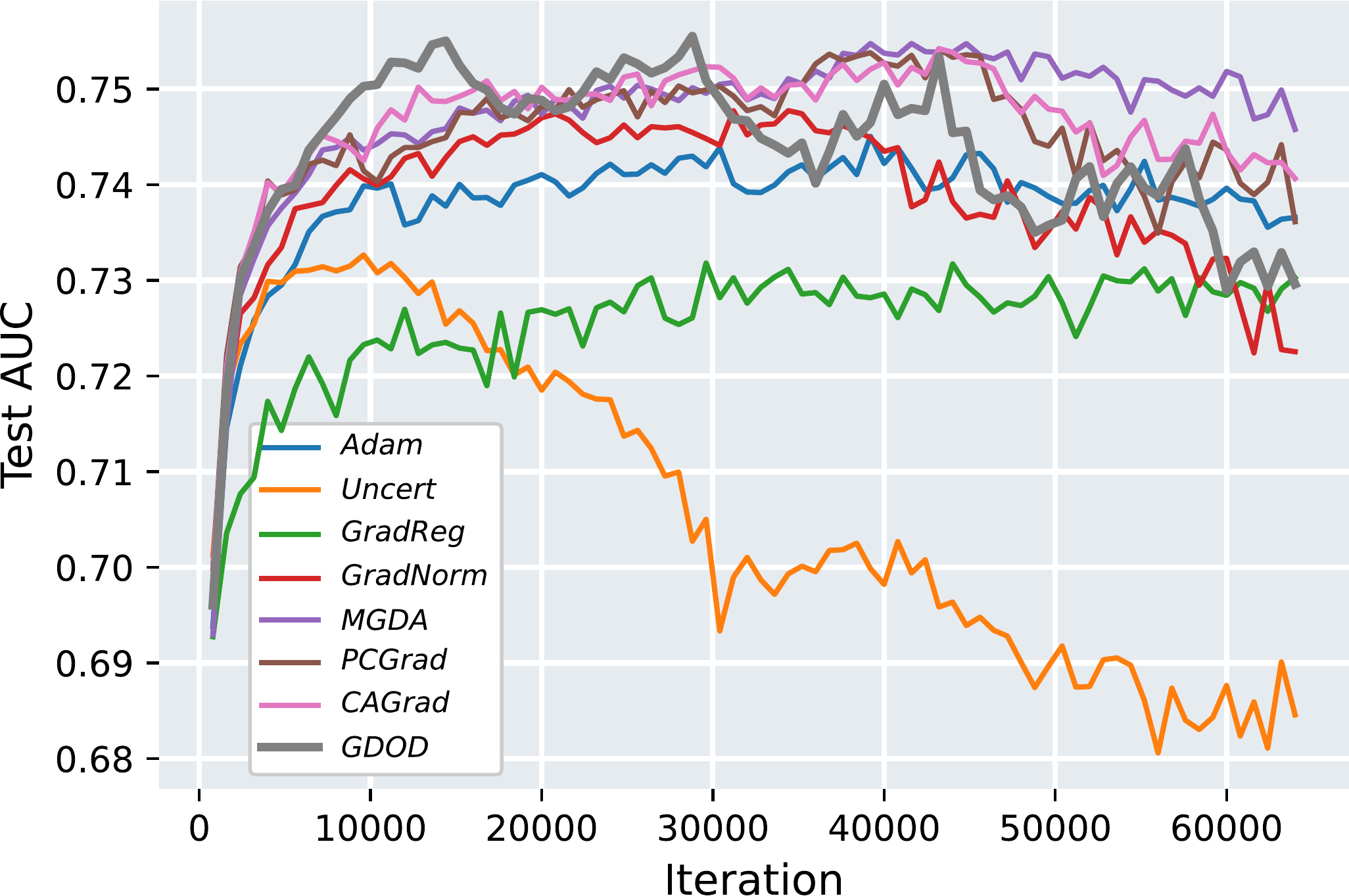}}
        
    \caption{Test AUC comparisons about several optimization methods on BookCrossing and IJCAI-15 datasets.
    In all cases GDOD outperforms all other optimization methods.}
    \label{fig:AUC Comparisons}
    \vspace{3pt}
\end{figure*}

\subsection{Optimization Method Comparison}~\label{sec:optimization}
Table~\ref{tab:commparison optimazation} shows the AUC of the comparative results on the BookCrossing dataset, IJCAI-15 dataset and Advertising dataset.
Focusing on the detail, the shared-bottom model combined with GDOD achieves higher AUC compared to other optimization methods.
Moreover, we have the following four observations: 1) GDOD, PCGrad and CAGrad outperform other five optimization methods.
This indicates that optimization methods manipulated per-task gradients are more practical.
2) GDOD achieves better performances than PCGrad and CAGrad, e.g., GDOD achieves 0.0064 and 0.0084 AUC gains compared to PCGrad and CAGrad in task2 with BookCrossing dataset respectively.
The magnitude of this improvement is fairly significant.
Because GDOD implements a decomposition method that can distinguish the conflicting gradients effectively.
3) Several situations with Uncert and GradReg are proven to be worse than Adam, showing the applicability of re-weighting methods is poor. 
4) MGDA seems to perform worse than some re-weighting methods in some tasks. 
This is because MGDA will attenuate the performance of tasks that have higher gradients.
Overall, these results verify that GDOD is a highly effective optimization method to avoid task competition.

Moreover, Figure~\ref{fig:Loss Comparisons} illustrates the test loss curves during the training procedure on BookCrossing and IJCAI-15 datasets.
From these curves, GDOD can be shown to achieve the lowest LogLoss than any other optimization method with a fixed step size.
Therefore, these results demonstrate that GDOD can accelerate convergence and achieve good performance at the same step.
We also show the AUC curves during the training procedure on the BookCrossing dataset and IJCAI-15 dataset in Figure~\ref{fig:AUC Comparisons}.
It is observed that GDOD achieves the highest AUC compared to all the other optimization methods.
Moreover, with a fixed training step, GDOD performs the best performance in most experiments.
These results demonstrate that GDOD outperforms other compared SOTA optimization methods.

\begin{table*}[h]
  \centering
  \captionsetup{font=bf}
  \caption{Performance of GDOD with MTL models on three datasets. The metrics are the average AUC and the average Logloss on the test dataset.
  }
\renewcommand\arraystretch{1.1}

\resizebox{507pt}{71pt}{
\begin{tabular}{ c c|c c | c c|c c | c c | c c | c c | c c}
 \toprule
 \multicolumn{2}{c}{\multirow{2}{*}{\textbf{Method}}}  &  \multicolumn{4}{|c|}{\textbf{BookCrossing} } & \multicolumn{4}{c|}{\textbf{IJCAI-15}} &  \multicolumn{6}{c}{\textbf{Alipay Advertising}}   \\
 \cline{3-16}
  \multicolumn{2}{c}{} & \multicolumn{2}{|c|}{Task1} & \multicolumn{2}{c}{Task2} & \multicolumn{2}{|c|}{Task1} & \multicolumn{2}{c}{Task2} & \multicolumn{2}{|c|}{Task1} & \multicolumn{2}{c}{Task2} & \multicolumn{2}{c}{Task3}  \\

 Basemodel & Optimizer & AUC & LogLoss & AUC & LogLoss & AUC & LogLoss & AUC & LogLoss & AUC & LogLoss & AUC & LogLoss & AUC & LogLoss \\
 \hline
 \hline
 \multirow{2}{*}{Shared-Bottom} 
& Adam & 0.7838 & 0.7068 & 0.7633 & 0.7251 & 0.6968 & 0.7257 & 0.7451 & 0.7187 & 0.7505 & 0.7096 & 0.7387 & 0.6994 & 0.8237 & 0.6992 \\
& GDOD & 0.7922 & 0.6969 & 0.7817 & 0.7211 & 0.7268 & 0.7231 & 0.7555 & 0.7159 & 0.7723 & 0.6991 & 0.7571 & 0.6951 & 0.8390 & 0.6956 \\

\hline
\multirow{2}{*}{Cross-Stitch} 
& Adam & 0.7771 & 0.7093 & 0.7543 & 0.7333 & 0.6932 & 0.7286 & 0.7421 & 0.7217 & 0.7518 & 0.7023 & 0.7436 & 0.6975 & 0.8256 & 0.6971 \\
& GDOD & 0.7954 & 0.6860 & 0.7830 & 0.7203 & 0.7114 & 0.7248 & 0.7567 & 0.7142 & 0.7708 & 0.6992 & 0.7577 & 0.6953 & 0.8384 & 0.6975 \\
 
 \hline
 \multirow{2}{*}{MMoE} 
& Adam & 0.7772 & 0.7097 & 0.7519 & 0.7386 & 0.6933 & 0.7302 & 0.7463 & 0.7182 & 0.7646 & 0.7030 & 0.7481 & 0.6954 & 0.8286 & 0.6964 \\
& GDOD & 0.7912 & 0.7019 & 0.7789 & 0.7216 & 0.7169 & 0.7243 & 0.7532 & 0.7157 & 0.7716 & 0.6997 & 0.7596 & 0.6944 & 0.8353 & 0.6973 \\
 
  \hline
 \multirow{2}{*}{PLE} 
& Adam & 0.7810 & 0.7078 & 0.7595 & 0.7300 & 0.6945 & 0.7315 & 0.7481 & 0.7173  & 0.7686 & 0.7060 & 0.7494 & 0.6950 & 0.8308 & 0.6973 \\
& GDOD & 0.7905 & 0.7032 & 0.7751 & 0.7243 & 0.7196 & 0.7239 & 0.7578 & 0.7132  & 0.7739 & 0.6989 & 0.7535 & 0.6968 & 0.8417 & 0.6955 \\
 
  \hline
 \multirow{2}{*}{SNR} 
& Adam & 0.7807 & 0.7067 & 0.7630 & 0.7278 & 0.7017 & 0.7278 & 0.7479 & 0.7180  & 0.7692 & 0.7014 & 0.7424 & 0.6985 & 0.8305 & 0.6967  \\
& GDOD & 0.7895 & 0.7033 & 0.7736 & 0.7246 & 0.7103 & 0.7251 & 0.7541 & 0.7161  & 0.7745 & 0.7003 & 0.7566 & 0.6955 & 0.8388 & 0.6959  \\
 \hline
\toprule
\end{tabular}
}
\label{tab: mtl_auc}
\end{table*}

\begin{table*}[htb]
  \centering
  \captionsetup{font=bf}
  \caption{AUC comparisons of different decomposition methods. 
  The baseline MTL model is Shared-Bottom. 
  \emph{Diff} measures the AUC gap between the decomposition method used in GDOD and other decomposition methods.
  }
\renewcommand\arraystretch{1.0}

\resizebox{507pt}{40pt}{
\begin{tabular}{ c |c c | c c|c c | c c|c c | c c | c c }
 \toprule
 \multicolumn{1}{c}{\multirow{3}{*}{\textbf{Decomposition Method}}}  &  \multicolumn{4}{|c|}{\textbf{BookCrossing} } & \multicolumn{4}{c|}{\textbf{IJCAI-15}} & \multicolumn{6}{c}{\textbf{Alipay Advertising}}  \\
 \cline{2-15}
  \multicolumn{1}{c}{} & \multicolumn{2}{|c|}{Task1} & \multicolumn{2}{c}{Task2} & \multicolumn{2}{|c|}{Task1} & \multicolumn{2}{c}{Task2} & \multicolumn{2}{|c|}{Task1} & \multicolumn{2}{c}{Task2} & \multicolumn{2}{|c}{Task3} \\
  & AUC & Diff & AUC & Diff & AUC & Diff & AUC & Diff & AUC & Diff & AUC & Diff & AUC & Diff \\
 \hline
 \hline
Random & 0.5897  & -0.2025  & 0.5848  & -0.1969  & 0.5623  & -0.1645  & 0.5477  & -0.2078  & 0.5584  & -0.2139  & 0.5434  & -0.2137  & 0.6303  & -0.2087 \\
QR & 0.7927  & 0.0005  & 0.7720  & -0.0097  & 0.7091  & -0.0177  & 0.7449  & -0.0106  & 0.7630  & -0.0093  & 0.7422  & -0.0149  & 0.8287  & -0.0103 \\\
RandDec & 0.7926  & 0.0004  & 0.7776  & -0.0041  & 0.7210  & -0.0058  & 0.7538  & -0.0017  & 0.7711  & -0.0012  & 0.7495  & -0.0076  & 0.8322  & -0.0068 \\
SVD & 0.7922  & - & 0.7817  & - & 0.7268  & - & 0.7555  & - & 0.7723  & - & 0.7571  & - & 0.8390  & - \\
\toprule
\end{tabular}
}
\label{tab: decomposition methods}
\end{table*}

\subsection{GDOD with Multi-task Models} 
\label{sec:gdod with mtl}
Table~\ref{tab: mtl_auc} shows the AUC and Logloss of the comparison results on BookCrossing, IJCAI-15 and Alipay Advertising datasets.
Focusing on the detail of Table~\ref{tab: mtl_auc}, all MTL models combined with GDOD achieve higher AUC and lower Logloss compared to the original MTL models. 
These results confirm that GDOD improves the performance for multi-task learning benchmarks by avoiding interference across all task gradients.
For example, the Cross-Stitch model with GDOD optimization achieves 0.0287 AUC gain compared to the original model in task2 with the BookCrossing dataset.
The magnitude of this improvement is fairly significant.
Moreover, some MTL networks also have addressed the negative transfer phenomenon. 
For example, PLE separates shared components and task-specific components and adopts a progressive routing mechanism to reduce negative transfer.
We can see that PLE outperforms other networks, such as MMOE and Cross-Stitch.
PLE with GDOD also achieves 0.01 AUC gain compared to the original model in most tasks.
It validates the effectiveness of GDOD and proves that mitigating conflicting gradients can boost the performance of MTL models.

\subsection{Ablation Study: Effect of Different Decomposition Methods}
\label{sec:decom compa}
In this section, we examine the effect of different decomposition methods in GDOD.
Our approach relies on the singular vectors from SVD to define the basis to identify the positive and negative components of each task gradients.
We compare SVD with several decomposition methods:
\begin{itemize}
    \item \textbf{Random} obtains the basis spanned by $r$ randomly chosen orthogonal vectors in $\mathbb{R}$.
    \item \textbf{QR Decomposition} is directly to decompose a matrix and seek the matrix column space as the orthogonal basis. Gram-Schmidt is a commonly used method to achieve this decomposition.
    \item \textbf{Randomized Approximate Matrix Decomposition (RandDec)}~\cite{halko2011finding} follows the framework that usually projects the original matrix to a low-rank sample space and then computes the approximate decomposition of the original matrix. 
\end{itemize}

Table~\ref{tab: decomposition methods} shows the AUC comparisons of different decomposition methods on BookCrossing, IJCAI-15 and Alipay Advertising datasets.
From Table~\ref{tab: decomposition methods}, we can see that SVD performs the best in most situations and Random achieves the worst performance.
We also observed a phenomenon that the magnitude of AUC \emph{diff} about task2 is greater than task1 in the BookCrossing dataset.
It demonstrates that a good choice of decomposition methods can mitigate the negative transfer across all tasks.

\begin{table*}[htb]
  \centering
  \captionsetup{font=bf}
  \caption{Performance comparisons of different optimization methods on Census-Income dataset. 
  The baseline MTL model is Shared-Bottom.
  }
\renewcommand\arraystretch{1.0}

\resizebox{500pt}{60pt}{
\begin{tabular}{ c|c c|c c|c c|c c|c c|c c }
 \toprule
 \multicolumn{1}{c}{\multirow{2}{*}{\textbf{Method}}}  & \multicolumn{2}{|c|}{Task1} & \multicolumn{2}{c}{Task2} & \multicolumn{2}{|c|}{Task3} & \multicolumn{2}{c}{Task4} & \multicolumn{2}{|c|}{Task5} & \multicolumn{2}{c}{Task6} \\
 ~ & AUC & Gain & AUC & Gain & AUC & Gain & AUC & Gain & AUC & Gain & AUC & Gain \\
 \hline
 \hline
 Adam & 0.93632 & ~ & 0.99350 & ~ & 0.90114 & ~ & 0.98416 & ~ & 0.83914 & ~ & 0.84135 & ~ \\
 Uncert & 0.93997 & 0.00365 & 0.99351 & 0.00001 & 0.90711 & 0.00597 & 0.98429 & 0.00013 & 0.84424 & 0.0051 & 0.85916 & 0.01781 \\
 GradReg & 0.94023 & 0.00391 & 0.99365 & 0.00015 & 0.90688 & 0.00574 & 0.98433 & 0.00017 & 0.85154 & 0.0124 & 0.86486 & 0.02351 \\
 GradNorm & 0.94307 & 0.00675 & 0.99378 & 0.00028 & 0.90657 & 0.00543 & 0.98455 & \textbf{0.00039} & 0.83941 & 0.00027 & 0.87282 & 0.03147 \\
  MGDA & 0.94060  & 0.00428  & 0.99390  & 0.00040  & 0.90480  & 0.00366  & 0.98440  & 0.00024  & 0.84690  & 0.00776  & 0.86000  & 0.01865 \\
 PCGrad & 0.93781 & 0.00149 & 0.99382 & 0.00032 & 0.90309 & 0.00195 & 0.98431 & 0.00015 & 0.84728 & 0.00814 & 0.8712 & 0.02985 \\
 CAGrad & 0.94145 & 0.00513 & 0.99415 & \textbf{0.00065} & 0.9067 & 0.00556 & 0.98437 & 0.00021 & 0.84979 & 0.01066 & 0.87646 & 0.03511 \\
 GDOD & 0.94328 & \textbf{0.00696} & 0.99408 & 0.00058 & 0.90794 & \textbf{0.00680} & 0.98429 & 0.00013 & 0.8504 & \textbf{0.01126} & 0.87659 & \textbf{0.03524} \\
 Weighted-GDOD & 0.94367 & \textbf{0.00735} & 0.99422 & \textbf{0.00072} & 0.90903 & \textbf{0.00789} & 0.98444 & \textbf{0.00028} & 0.85062 & \textbf{0.01148} & 0.88536 & \textbf{0.04401} \\

\hline
\toprule
\end{tabular}
}
\label{tab:wgdod}
\end{table*}

\begin{figure}
    \centering
    \captionsetup{font=bf}
    \subfigure[AUC on BookCrossing]{
        \includegraphics[width=115pt, trim=0pt 0pt 0pt 0pt, clip]{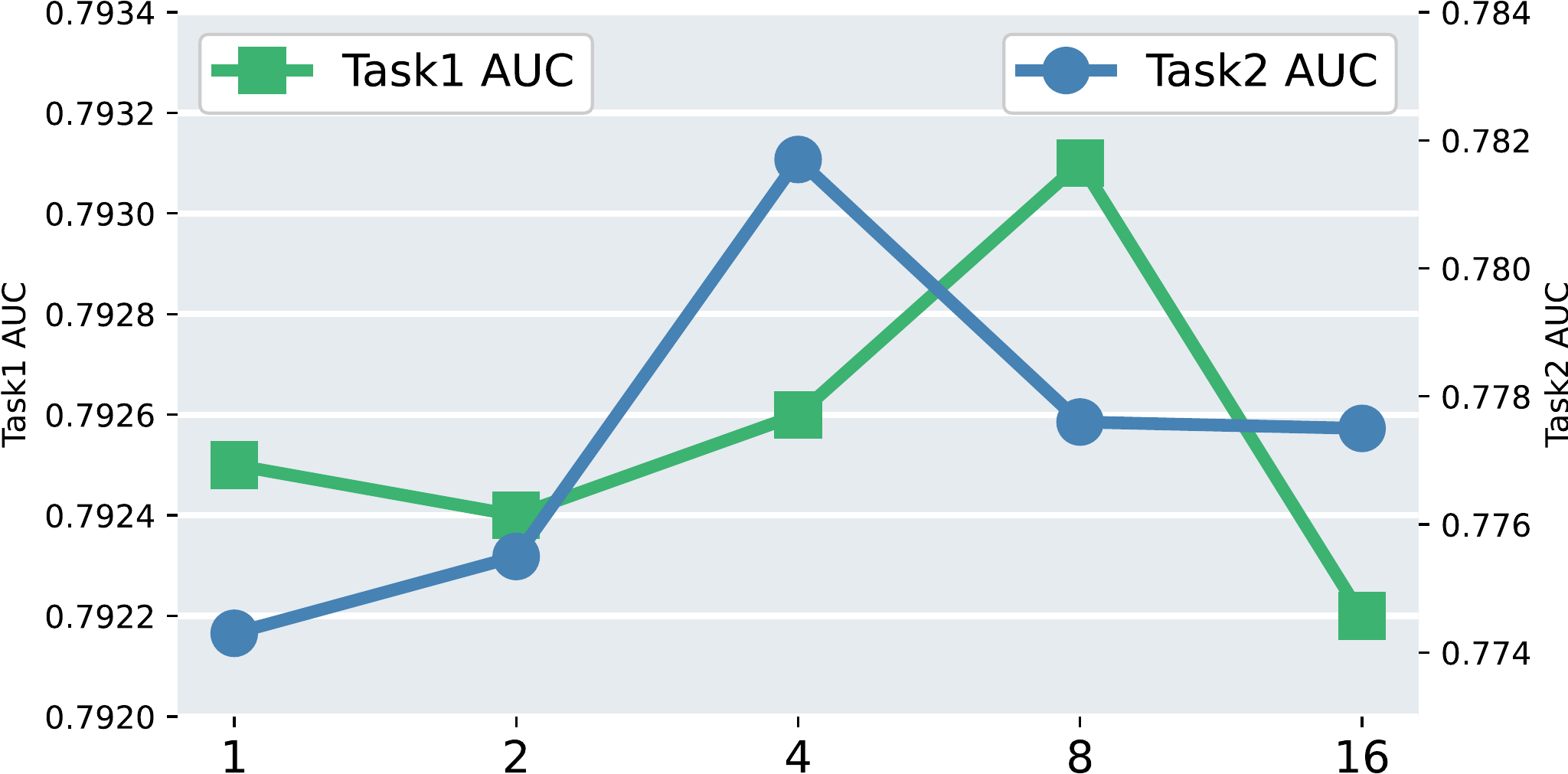}
        \label{subfig:book}}
    \subfigure[AUC on IJCAI-15]{
        \includegraphics[width=115pt, trim=0pt 0pt 0pt 0pt, clip]{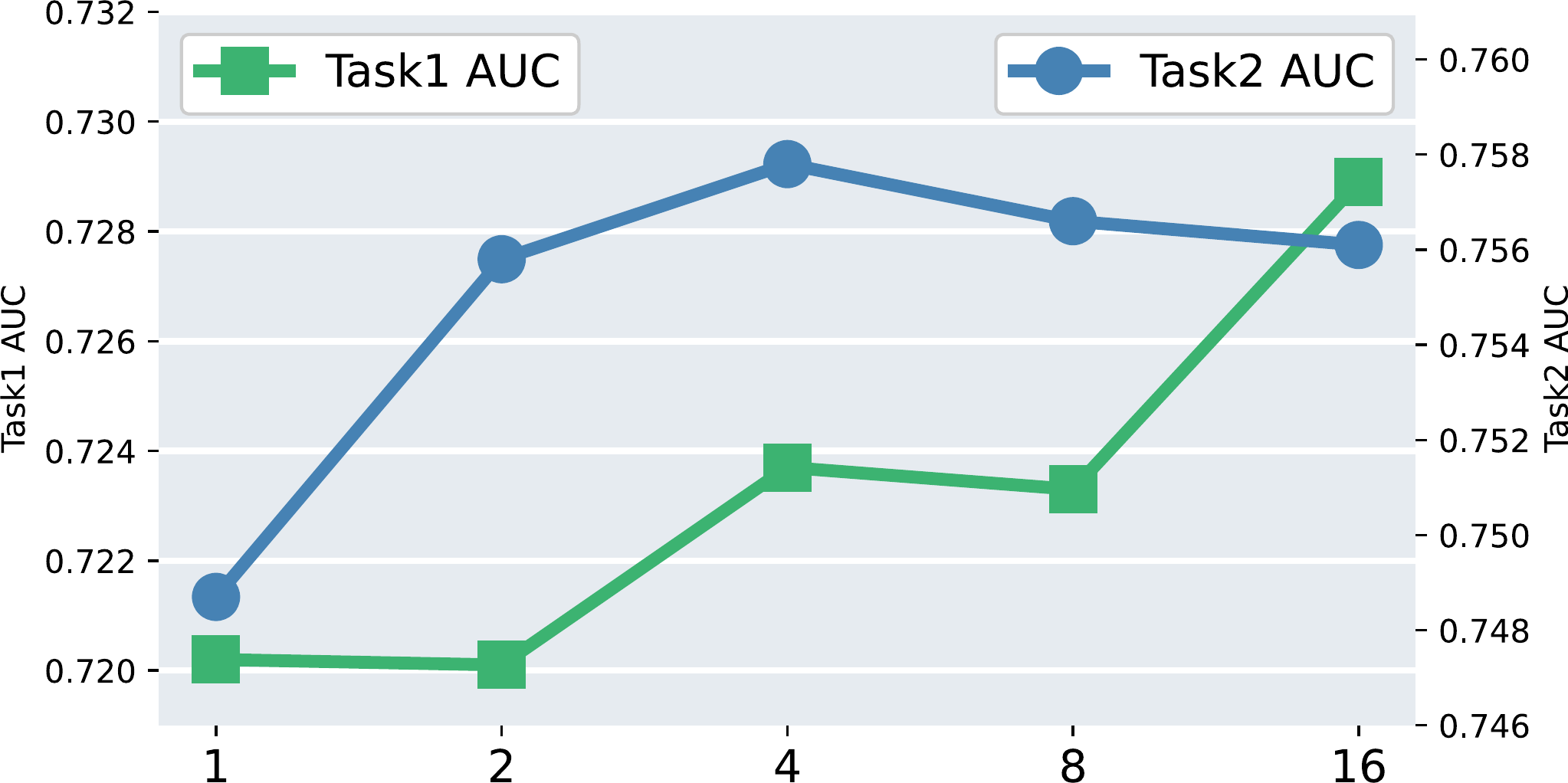}
        \label{subfig:ijcai}}
    \caption{Performance with different dimensions of subspace $\mathcal{S}$ on BookCrossing and IJCAI-15 datasets. 
    }
    \label{fig:size}
\end{figure}

\subsection{Ablation Study: Effect of Different Dimensions of Subspace $\mathcal{S}$}
\label{sec:size of S}
In this section, we examine the effect of different dimensions of subspace $\mathcal{S}$ in GDOD.
Figure~\ref{fig:size} depicts the task AUC varies with different dimension of the subspace $\mathcal{S}$ on BookCrossing and IJCAI-15 datasets.
From Figure~\ref{subfig:ijcai}, we observe that it is better to decompose all the task gradients in a larger dimensional subspace.
In general, a larger dimensional subspace possibly captures a richer description of the matrix $\mathcal{M}$.
However, Figure~\ref{subfig:book} holds the opposite phenomenon.
This is because a larger dimensional also creates the risk of over-fitting especially in a limited dataset, such as the Bookcrossing dataset.

\begin{figure}
    \centering
    \subfigure[Task1 on BookCrossing]{
        \includegraphics[width=115pt, trim=0pt 0pt 0pt 0pt, clip]{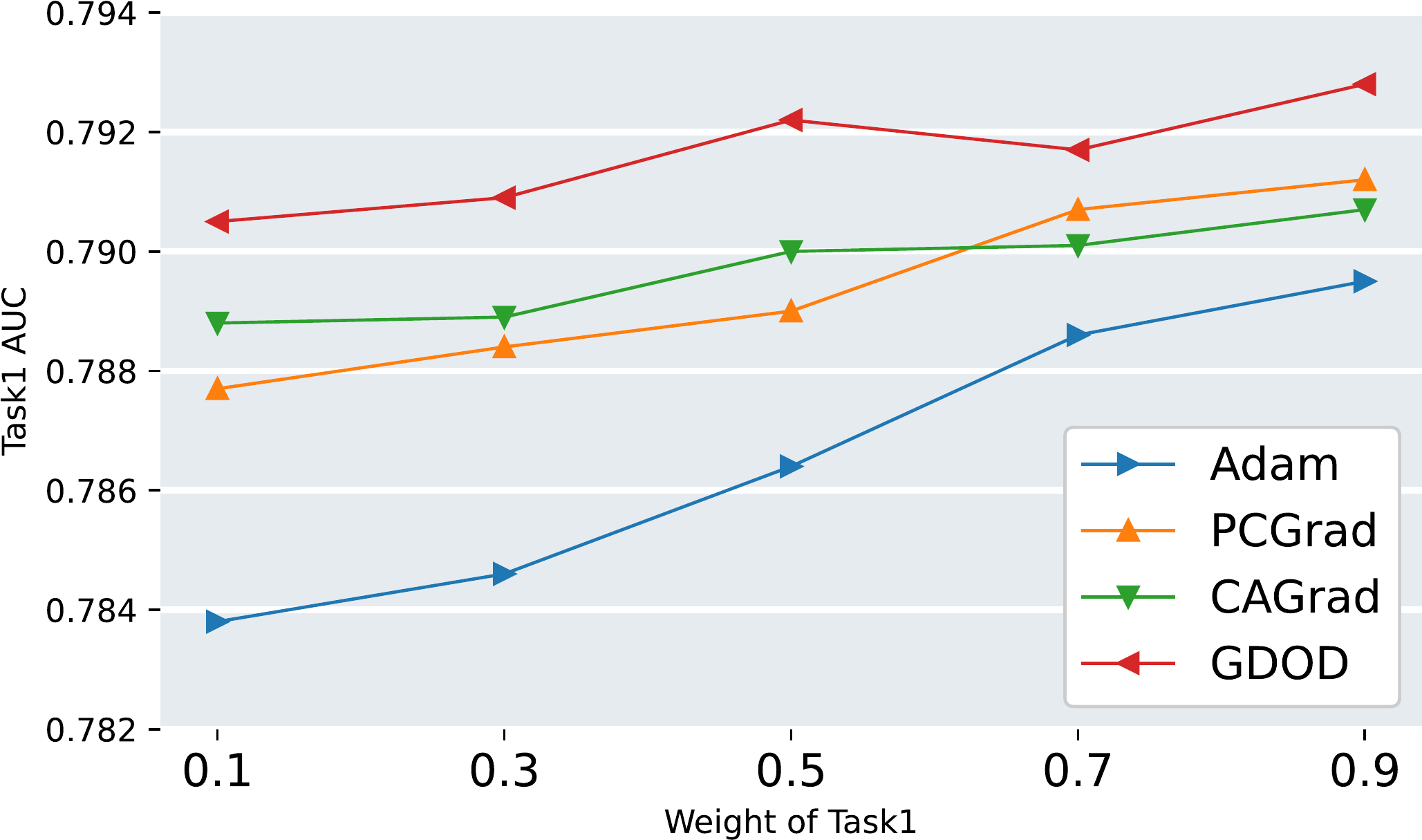}}
    \subfigure[Task2 on BookCrossing]{
        \includegraphics[width=115pt, trim=0pt 0pt 0pt 0pt, clip]{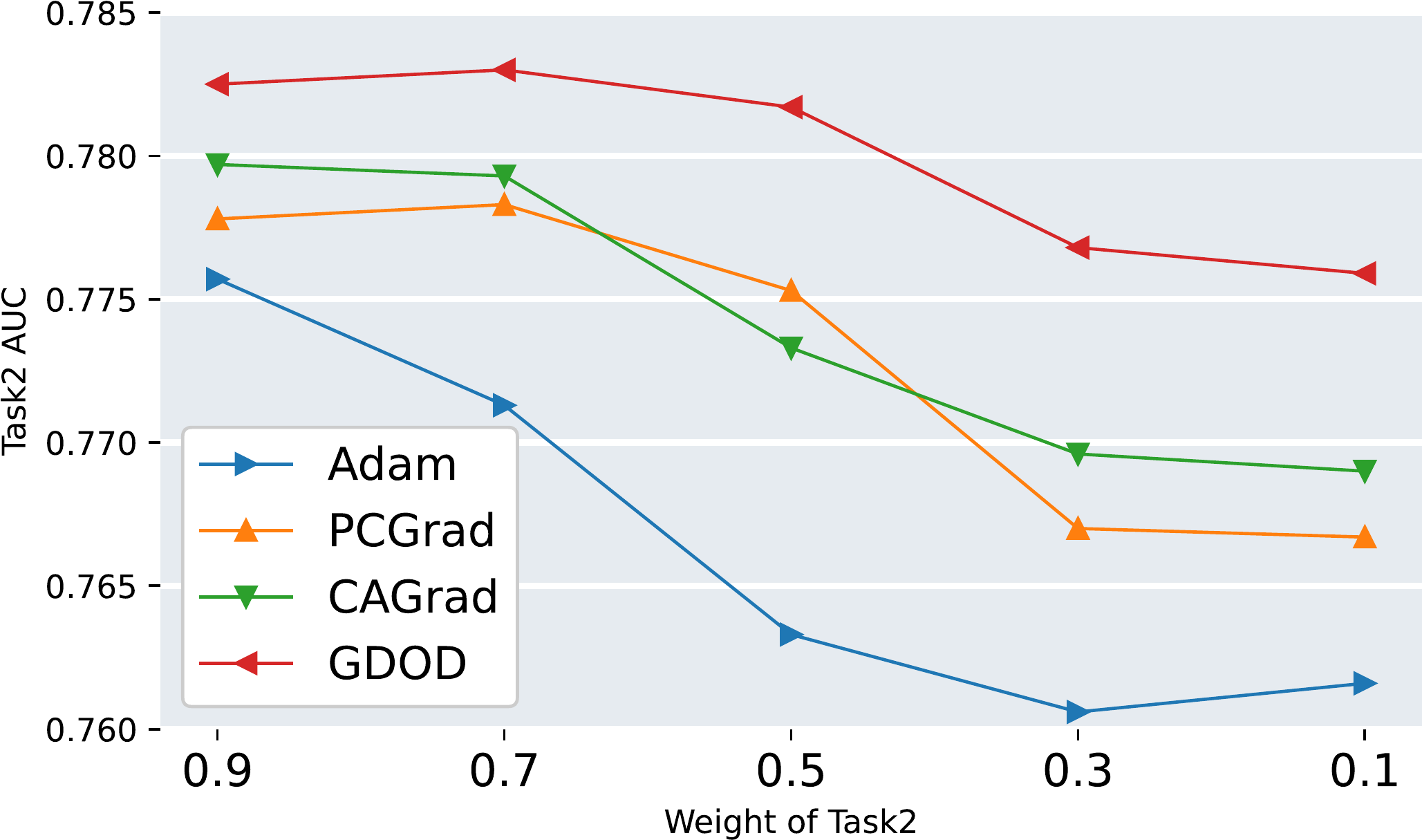}}
    \quad
    \subfigure[Task1 on IJCAI-15]{
        \includegraphics[width=115pt, trim=0pt 0pt 0pt 0pt, clip]{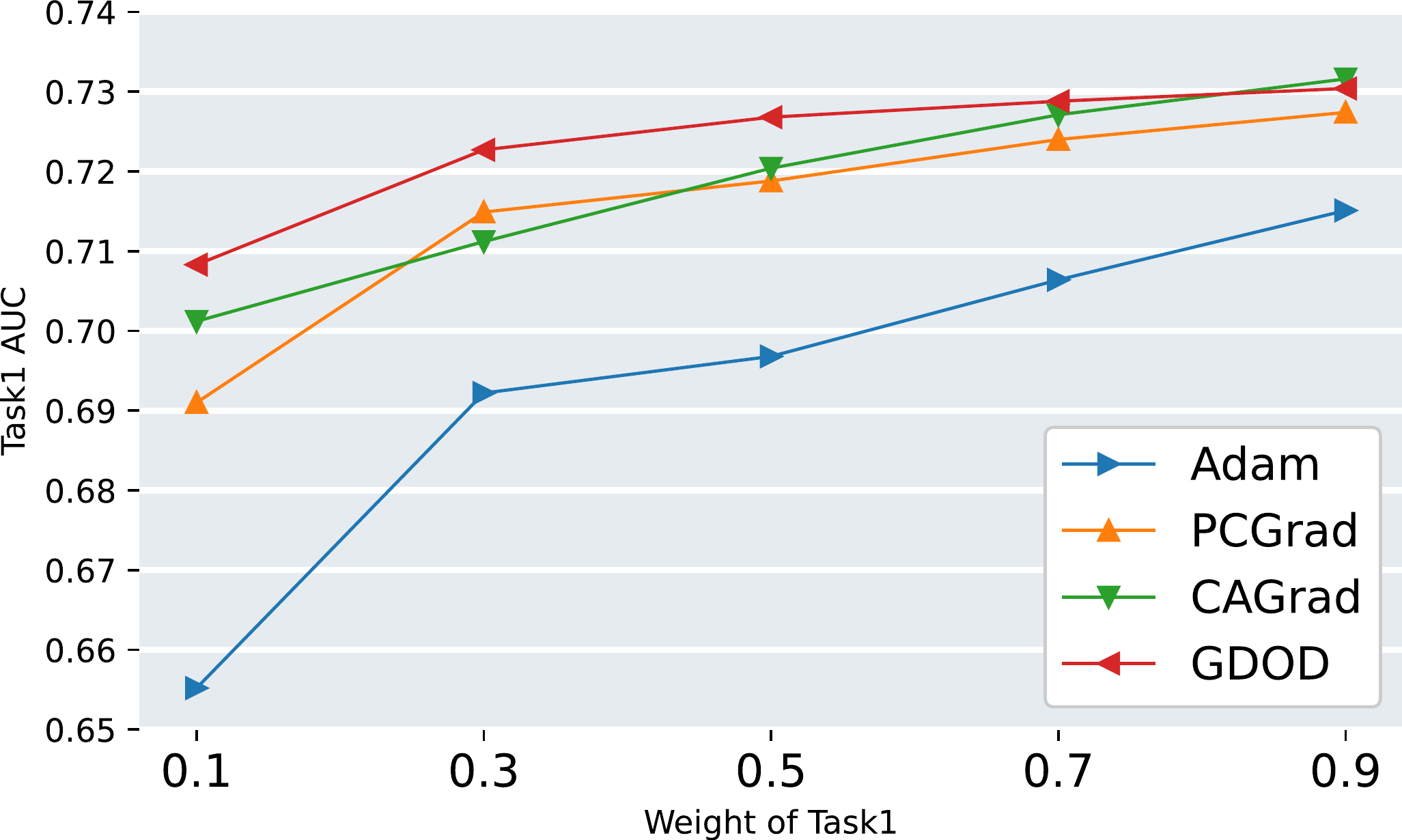}}
    \subfigure[Task2 on IJCAI-15]{
        \includegraphics[width=115pt, trim=0pt 0pt 0pt 0pt, clip]{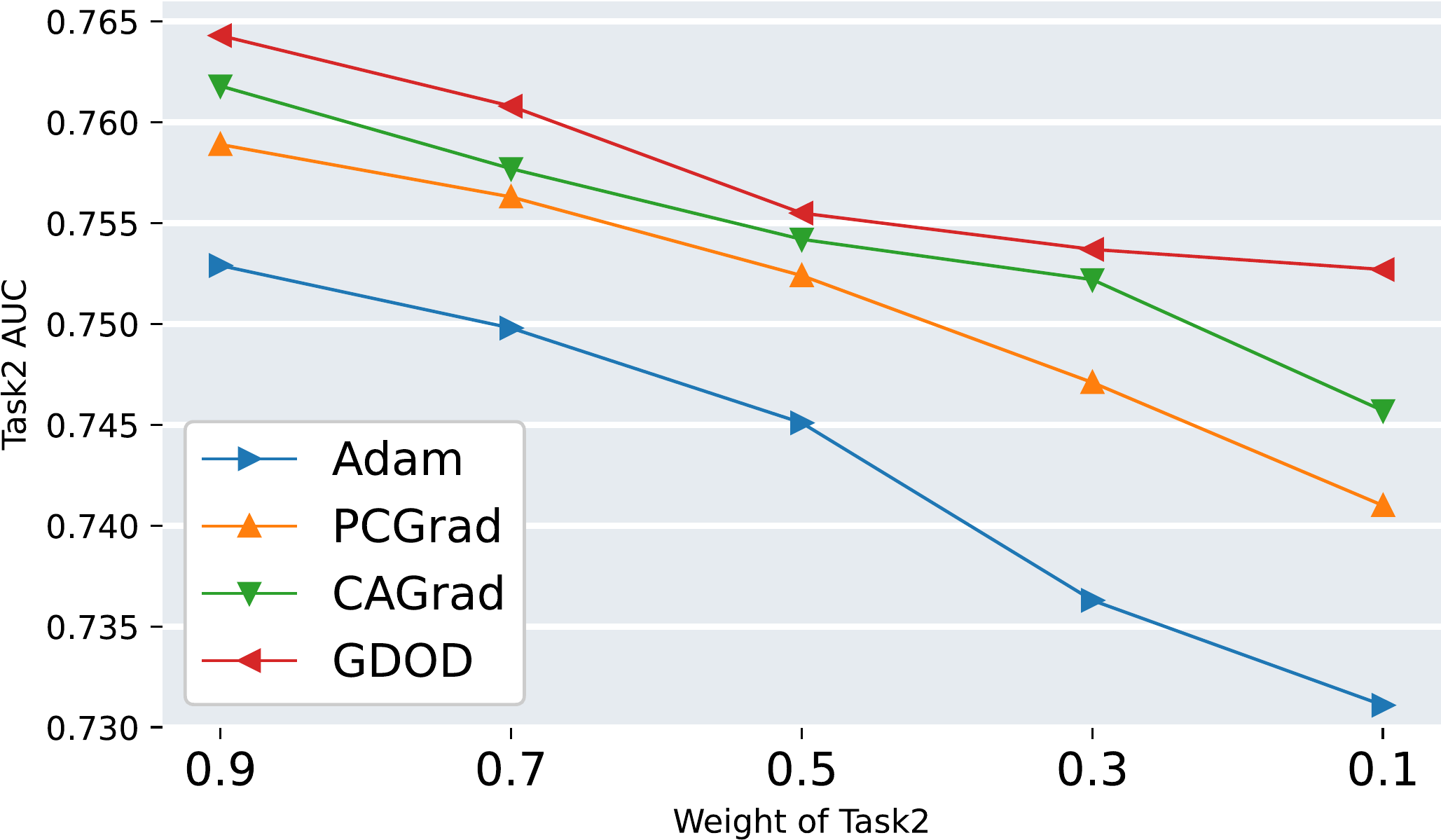}}
    \caption{Methods comparison with different weights of tasks. The sum of the weights of the two tasks is one.}
    \label{fig: varying weights}
\end{figure}

\subsection{Ablation Study: Effect of Tasks with Varying Weights}
\label{sec:uncer task weight}
In this section, we examine the effect with varying weights for each tasks.
Figure~\ref{fig: varying weights} shows AUC varies with different weights for each task on BookCrossing and IJCAI-15 datasets with several gradient-based methods.
The weight for each task is equally in the previous setting.
From Figure~\ref{fig: varying weights}, we can see that a task with a higher weight indication
probability usually receives a higher AUC.
It is obviously that GDOD performs the best with varying task weights in most situations.
Moreover, for CAGrad, the performance of a task with a smaller weight (the weight for task 2 is 0.1) reduces significantly.
This is because that CAGrad searches the new updated vector is around $g_0$ (average gradient vector). 
However, the reduction for GDOD is smaller than other methods.
It verify that GDOD is a more robust algorithm.

\subsection{GDOD with More Tasks}
In Algorithm~\ref{alg:ogd}, GDOD uses the helpful components which refer to the projections of original gradients onto the basis vectors where all task gradients agree in the direction to update the model parameters.
However, as the number of tasks increases, the components of all tasks in the same direction will decrease.
To deal with more tasks, we propose a weighted-GDOD which defines a weight for task components from the dimension of basis.
For each basis vector, the gradient components of all tasks are divided into two sets $\{S^+\}$ and $\{S^-\}$ by the sign.
Suppose $\{S^+\}$ and $\{S^-\}$ have $a$ and $b$ gradient components respectively. 
The weight for each gradient component is calculated as following:
\begin{itemize}
    \item If $a \geq b$, the weights for gradient components in set $S^+$ and $S^-$ are $\frac{a-b}{K}$ and 0.
    \item If $a \textless b$, the weights for gradient components in set $S^+$ and $S^-$ are 0 and $\frac{b-a}{K}$.
\end{itemize}

We also examine the effect of weighted-GDOD with the Census-Income dataset that has six tasks.
As shown in Table~\ref{tab:wgdod}, we observe that weighted-GDOD and GDOD achieve the best performance in most tasks.
Especially, weighted-GDOD and GDOD realize significant improvements for task 5 and 6.
Moreover, all results with weighted-GDOD are proven to be better than GDOD, showing that GDOD with a weighted policy is more effective with more tasks.

\section{Conclusion}
In this paper, we present a novel optimization approach for MTL, GDOD, which manipulates each task gradient using a decomposition built from the span of all task gradients. 
GDOD decomposes gradients into task-shared and task-specific components explicitly and adopts a general update rule for avoiding interference across all task gradients. 
Moreover, we present the convergence of GDOD theoretically under both convex and non-convex assumptions. 
Experiment results on several multi-task datasets not only demonstrate the significant improvement of GDOD performed to existing MTL models but also outperform state-of-the-art optimization methods in terms of AUC metric.
Our future study would focus on exploring other decomposition methods to optimize training procedure for more effective and efficient multi-task learning.
\bibliographystyle{ACM-Reference-Format}
\balance
\bibliography{reference}


\end{document}